%% file: main.tex
\documentclass[11pt]{article}
\usepackage[margin=1.0in]{geometry}
\usepackage{newpxmath}
\usepackage{times}
\usepackage[table]{xcolor}
\usepackage{url}    
\usepackage{tabularx}
\usepackage{booktabs}             
\usepackage{nicefrac}       
\usepackage{microtype}     
\usepackage{mathtools}
\usepackage{natbib}
\usepackage{graphicx}
\usepackage{prettyref}
\usepackage{textcomp}
\usepackage{mathrsfs}
 
\usepackage{amsthm}
\usepackage{thmtools,thm-restate}
\usepackage{amsbsy}

\usepackage{amssymb}
\usepackage{algorithmic}
\usepackage[]{algorithm}
\usepackage{multirow}
\usepackage{enumitem}
\allowdisplaybreaks[3]
\usepackage{subcaption}
\usepackage{stmaryrd}
\usepackage{textcomp}
\usepackage{relsize}
\usepackage{todonotes}
\usepackage[pagebackref, colorlinks=true, linkcolor=teal, citecolor=teal, urlcolor=teal]{hyperref}
\usepackage{cleveref}

\usepackage{color-edits}
\addauthor[Vatsal]{vs}{red}
\addauthor[HL]{hl}{magenta}
\addauthor[Lunjia]{lunjia}{blue}

\author{
  Lunjia Hu\thanks{Northeastern University. \texttt{lunjia@alumni.stanford.edu}}
  \and
  Haipeng Luo\thanks{University of Southern California. \texttt{haipengl@usc.edu}}
  \and
  Spandan Senapati\thanks{University of Southern California. \texttt{ssenapat@usc.edu}}
  \and
  Vatsal Sharan\thanks{University of Southern California. \texttt{vsharan@usc.edu}}
}

\include{command_preamble}

\title{Efficient Swap Multicalibration of Elicitable Properties\thanks{Author ordering is alphabetical.}}
\begin{document}
\date{}

\maketitle
\thispagestyle{empty}

\begin{abstract} 
    Multicalibration \citep{hebert2018multicalibration} is an algorithmic fairness perspective which demands that the predictions of a predictor are correct conditional on themselves and membership in a collection of potentially overlapping subgroups of a population. The work of \cite{noarov2023statistical} established a surprising connection between multicalibration for an arbitrary property $\Gamma$ (e.g., mean or median) and property elicitation: a property $\Gamma$ can be multicalibrated if and only if it is elicitable, where elicitability is the notion that the true property value of a distribution can be obtained by solving a regression problem over the distribution. In the adversarial (online) setting, \cite{noarov2023statistical} proposed an inefficient algorithm that achieves $\tilde{\cO}(\sqrt{T})$ $\ell_{2}$-multicalibration error for a hypothesis class of group membership functions and an elicitable property $\Gamma$, after $T$ rounds of interaction between a forecaster and adversary. 
    
    In this paper, we generalize multicalibration for an elicitable property $\Gamma$ from group membership functions to arbitrary bounded hypothesis classes and introduce a stronger notion --- swap multicalibration, following \cite{gopalan2023swap}. Subsequently, we propose an oracle-efficient algorithm which when given access to an online agnostic learner, achieves $\tilde{\cO}(T^{\frac{1}{r + 1}})$ $\ell_r$-swap multicalibration error with high probability $(r \ge 2)$ for a hypothesis class with bounded sequential Rademacher complexity and an elicitable property $\Gamma$. For the special case of $r = 2$, this implies an oracle-efficient algorithm that achieves $\tilde{\cO}(T^{\frac{1}{3}})$ $\ell_{2}$-swap multicalibration error, which significantly improves on the previously established bounds for the problem \citep{noarov2023statistical, ghugeimproved, luo2025improved}, and completely resolves an open question raised in \cite{garg2024oracle} on the possibility of an oracle-efficient algorithm that achieves $\tilde{\cO}(\sqrt{T})$ $\ell_{2}$-mean multicalibration error by answering it in a strongly affirmative sense.
\end{abstract}

\newpage
\setcounter{page}{1}

\section{Introduction}\label{sec:introduction}
With the rise of machine learning models in several real-world applications such as healthcare, medical support, law, finance, and safety-critical control, quantifying the model's confidence in its predictions is becoming an increasingly serious problem. A widely studied metric to quantify the uncertainty in a model's predictions is \emph{calibration} \citep{dawid1982well}, which demands that the model's predictions are correct conditional on themselves. To define calibration more formally, we set up some notation. Let $\cX$ be the instance space, $\cY = \{0, 1\}$ be the label space, $D \in \Delta(\cX \times \cY)$ be an unknown distribution over $\cX \times \cY$, and $p: \cX \to [0, 1]$ be a predictor. The predictor $p$ is (mean) calibrated if $\mathbb{E}_{(x, y) \sim D}[y|p(x) = v] = v$ for all $v \in \text{Range}(p)$. Calibration ensures a certain level of trustworthiness since the predictions are a true representation of the ground truth. However, calibration only ensures unbiasedness at the scale of the level sets of the predictor and cannot address broader reliability and fairness concerns, e.g., when the instance space $\cX$ also consists of several refinements corresponding to potentially overlapping demographic groups, and the model's predictions should not lead to any detectable bias among the groups. To incentivize predictors to be unbiased not just marginally, but also for a collection of protected subgroups of a population, \cite{hebert2018multicalibration} introduced \emph{multicalibration}. Given a set $\cS$ of potentially overlapping subsets of $\cX$, the predictor $p$ is (mean) multicalibrated if \begin{align*}
    \mathbb{E}_{(x, y) \sim D} [y| p(x) = v, x \in s] = v, \text{for all } v \in \text{Range}(p), s \in \cS.
\end{align*}

Although multicalibration was initially proposed in the distributional (offline) setting by \cite{hebert2018multicalibration} as a mechanism to incentivize fair predictions, recent years have witnessed surprising connections with several different areas, e.g., complexity theory \citep{casacuberta2024complexity,global-calibration}, algorithmic fairness \citep{obermeyer2019dissecting}, learning theory \citep{gopalan_et_al:LIPIcs.ITCS.2022.79, gopalan2022low}, indistinguishability and pseudo-randomness \citep{dwork2021outcome, gopalan2022loss,pseudoentropy,supersimulators}, and so on, leading to a surge of interest in the community. Despite its conceptual importance and broad influence, achieving multicalibration in a sample-efficient manner has remained an outstanding problem, with existing constructions \citep{hebert2018multicalibration, gopalan_et_al:LIPIcs.ITCS.2022.79, globus2023multicalibration} requiring $\gtrsim \varepsilon ^ {-10}$ samples to efficiently learn a predictor with multicalibration error at most $\varepsilon$. The above problem also manifests in the more challenging adversarial (online) setting, where an online forecaster and adversary interact sequentially for $T$ rounds, and the goal of the forecaster is to minimize the multicalibration error at the end of their interaction. Notably, existing online multicalibration algorithms suffer from one or more of the following issues: (a) they achieve error rates that scale poorly with $T$; (b) are inefficient; (c) offer strong theoretical guarantees only under specific settings; and (d) do not readily generalize beyond multicalibration for means ({we define multicalibration for an arbitrary property in \pref{subsec:property_elicitation}}).
Although in recent years, online multicalibration has gained significant attention \citep{gupta2022online, haghtalab2023unifying, garg2024oracle, noarovhigh, luo2025improved, ghugeimproved}, the aforementioned issues remain inherent to the proposed algorithms. The motivating goal of this paper is to devise algorithms that simultaneously circumvent the challenges posed above.

\subsection{Online (Mean) Multicalibration}
Motivated by the above concerns, in this paper, we primarily consider online multicalibration  --- a sequential decision-making problem between an online forecaster and an adversary over $T$ rounds of interaction. In online (mean) multicalibration, at each time $t \in [T]$, (a) the adversary reveals a context $x_{t} \in \cX$, where $\cX$ is the instance space; (b) the forecaster predicts a distribution $p_{t} \in [0, 1]$ over binary outcomes; (c) the adversary reveals the true label $y_{t} \in \cY = \{0, 1\}$. For a bounded hypothesis class $\cF \subset [-1, 1] ^ {\cX}$, the set of forecasts $p_{1}, \dots, p_{T}$ is multicalibrated if  no hypothesis $f$ in $\cF$ is marginally able to detect any correlation $f(x_t) (y_t - p_t)$ with the residual error $y_{t} - p_{t}$ when conditioned on the rounds where the prediction made is $p_t = p$. Mathematically, this is quantified by minimizing the $\ell_{r}$-multicalibration error $(r \ge 1)$, defined as \begin{align}\label{eq:mcal}\mcal_{\text{mean}, r}(\cF) \coloneqq \sup_{f \in \cF}\sum_{p \in [0, 1]} n_p \abs{\frac{1}{n_p}\sum_{t = 1} ^ {T} \ind{p_t = p} f(x_t) (y_t - p_t)} ^ {r},\end{align} where $n_{p} \coloneqq \sum_{t = 1} ^ {T} \ind{p_t = p}$ denotes the number of rounds where the prediction made is $p_t = p$. Note that the summation $\sum_{p \in [0, 1]}$ is an abuse of notation and that it is really over the (at most $T$) distinct predictions made by the forecaster. Online calibration ($\cal_{\text{mean}, r}$) is a special case of multicalibration where $\cF = \{1\}$ and $1$ denotes the constant function that evaluates to $1$ for all $x \in \cX$:
\begin{align*}
    \cal_{\text{mean}, r} \coloneqq \sum_{p \in [0, 1]} n_{p} \abs{\frac{1}{n_p} \sum_{t = 1} ^ {T} \ind{p_t = p} (y_t - p_t)} ^ {r}.
\end{align*}
Throughout the paper, we consider forecasters that make predictions in a finite discretization $\cZ = \{z_{1}, \dots, z_{N}\}$ of $[0, 1]$, where $z_{i} = \frac{i}{N}$ and $N$ is a parameter to be specified later. We let $n_{i}$ be a shorthand for $n_{z_i}$.

In a line of work initiated to understand the implications of multicalibration towards \emph{omniprediction} --- a simultaneous loss-minimization paradigm introduced by \cite{gopalan_et_al:LIPIcs.ITCS.2022.79}, \cite{gopalan2023swap} introduced a related notion called \emph{swap multicalibration}. Informally, while multicalibration requires a marginal (over $p$) guarantee for the conditional correlation errors, swap multicalibration demands that the conditional correlation error for each $p$ be small. Formally, the $\ell_r$-swap multicalibration error incurred by the forecaster is defined as \begin{align}\label{eq:swap_multical_comparator}
    \smcal_{\text{mean}, r}(\cF) \coloneqq \sum_{p \in [0, 1]}  n_{p}\sup_{f \in \cF} \abs{\frac{1}{n_p}\sum_{t = 1} ^ {T} \ind{p_t = p} f(x_t) (y_t - p_t)} ^ {r}.
\end{align}
Compared to multicalibration, swap multicalibration is a stronger notion which requires that for all $p \in [0, 1]$, even when conditioned on the rounds where the prediction made is $p_t = p$, the no correlation guarantee holds for each $f \in \cF$. Clearly, we have $\smcal_{\text{mean}, r}(\cF) \ge \mcal_{\text{mean}, r}(\cF)$.

\subsection{A Connection between Multicalibration and Property Elicitation}\label{subsec:property_elicitation}
Although conventionally stated in the context of means,  mutlicalibration can be defined more generally for any distributional property $\Gamma$, e.g., median, variance and higher $k$-th order $(k \ge 2)$ moments, $q$-th quantiles, etc. In the distributional setting, let $D \in \Delta(\cX \times \cY)$ be an unknown distribution, where, as usual, $\cX$ denotes the instance space and $\cY \subseteq [0, 1]$ represents the label space (not necessarily $\{0, 1\}$). {Let $\cS$ be a collection of potentially overlapping subsets of $\cX$} and $\Gamma: \cP \to [0, 1]$ be a property ({without any loss of generality, we assume $\text{Range}(\Gamma) = [0, 1]$}), where $\cP$ is a family of probability distributions over $\cY$. A predictor {$\gamma: \cX \to [0, 1]$} is \emph{$\Gamma$-multicalibrated} if $\Gamma(D|\gamma(x) = v, x \in s) = v$ holds for all $v \in [0, 1], s \in \cS$.
In other words, the predictions of $\gamma$ about the property $\Gamma$ are correct even when conditioned on the level sets of $\gamma$ and membership in a collection of potentially overlapping subsets of $\cX$. Similarly, \emph{$\Gamma$-calibration} is the marginal requirement that $\Gamma(D| \gamma(x) = v) = v$ for all $v \in [0, 1]$. \cite{jung2021moment} initiated the study of moment multicalibration and showed that it is impossible to multicalibrate variance and other higher moments since there exists a simple distribution $D$ on which even the true distributional predictor for variance $\gamma_{\text{var}}(x) \coloneqq \mathbb{E}[(y - \mathbb{E}[y | x]) ^ {2} | x]$ is not calibrated.\footnote{{To observe this, let $D$ be supported on $(x_{1}, 1)$ and $(x_{2}, 0)$ with $\mathbb{P}(x = x_1) = \mathbb{P}(x = x_2) = 0.5$. Then the true distributional predictor for variance $\gamma_{\text{var}}$ predicts $\gamma_{\text{var}}(x_1) = \gamma_{\text{var}}(x_2) = 0$, where as $\Gamma(D| \gamma_{\text{var}}(x) = 0) = 0.5$ is the variance of a $\text{Bernoulli}(0.5)$ random variable. Compare this to mean multicalibration, where the true distributional predictor $\gamma_{\text{mean}}(x) \coloneqq \mathbb{E}[y|x]$ is always calibrated. As per \cite{dwork2021outcome}, a predictor $\gamma$ is $\Gamma$-multicalibrated if it is indistinguishable from the true distributional predictor $\gamma_{\Gamma}(x) \coloneqq \Gamma(D|x)$ based on statistical tests defined by $\cG$, therefore, it is necessary for $\gamma_{\Gamma}$ to be calibrated.}} 
Subsequent work by \cite{noarov2023statistical} established a remarkable connection between multicalibration and property elicitation \citep{lambert2008eliciting, steinwart2014elicitation}. Formally, a property $\Gamma$ is \emph{elicitable} if there exists a loss function {$\ell: [0, 1] \times \cY \to \mathbb{R}$} whose minimizer in expectation over $P$ yields the true property value $\Gamma (P)$, i.e., {$\Gamma(P) \in \argmin_{\gamma \in [0, 1]} \mathbb{E}_{y \sim P}[\ell(\gamma, y)]$}  for all $P \in \cP$ (we say $\ell$ elicits $\Gamma$). For instance, the squared loss $\ell(\gamma, y) = \frac{1}{2}(\gamma - y)^{2}$ elicits the mean, where as the $\ell_{1}$-loss $\ell(\gamma, y) = \frac{1}{2}\abs{\gamma - y}$ elicits the median; {more generally, the \emph{pinball} loss $\ell(\gamma, y) = (\gamma - y)\bigc{(1 - q)  \ind{y \le \gamma} - q \ind{y > \gamma}}$ elicits the $q$-th quantile for $q \in [0, 1]$ (median corresponds to $q = 0.5$)}. A striking result due to \cite{noarov2023statistical} is that a property $\Gamma$ can be multicalibrated if and only if $\Gamma$ is elicitable. In both the distributional and online settings, \cite{noarov2023statistical} propose canonical algorithms for achieving multicalibration for an elicitable property $\Gamma$. {The result of \cite{noarov2023statistical} builds on the following key characterization of property elicitation established by \cite{steinwart2014elicitation}: a property $\Gamma$ is elicitable if and only if there exists an \emph{identification} function for $\Gamma$ --- a function {$V: [0, 1] \times \cY \to \mathbb{R}$} such that $\mathbb{E}_{y \sim P} [V(\gamma, y)] = 0 \iff \gamma = \Gamma(P)$ for all $P \in \cP$. In other words, the true property value $\Gamma(P)$ is a root of the expected identification function for all $P \in \cP$, e.g., for mean $V(\gamma, y) = \gamma - y$, for median $V(\gamma, y) = \ind{y \le \gamma} - \frac{1}{2}$, {where as for the $q$-th quantile $V(\gamma, y) = \ind{y \le \gamma} - q$}. More generally, for an elicitable property $\Gamma$, $V(\gamma, y) = \frac{\partial \ell(\gamma, y)}{\partial \gamma}$, where $\ell$ is the loss function that elicits $\Gamma$. Based on the concept of an identification function, we generalize the definition of online multicalibration for an elicitable property beyond group membership functions and introduce swap multicalibration. Specifically, for an elicitable property $\Gamma$ with identification function $V$ ({without any loss of generality, we assume $\text{Range}(V) = [-1, 1]$ since the loss that elicits a property can be scaled appropriately}), $r \ge 1$, and a hypothesis class $\cF \subset [-1, 1] ^ {\cX}$, we define the $\ell_{r}$-multicalibration and $\ell_{r}$-swap multicalibration errors incurred by the forecaster respectively as: 
\begin{align}
    \mcal_{\Gamma, r}(\cF) & \coloneqq \sup_{f \in \cF} \sum_{p \in [0, 1]} n_{p} \abs{\frac{1}{n_{p}} \sum_{t = 1} ^ {T} \ind{p_{t} = p} f(x_{t}) V(p_{t}, y_{t})} ^ {r}, \label{eq:mcal_def_intro} \\
    \smcal_{\Gamma, r}(\cF) & \coloneqq  \sum_{p \in [0, 1]} n_{p} \sup_{f \in \cF} \abs{\frac{1}{n_{p}} \sum_{t = 1} ^ {T} \ind{p_{t} = p} f(x_{t}) V(p_{t}, y_{t})} ^ {r}. \label{eq:smcal_def_intro} 
\end{align}
Calibration with a property $\Gamma$ ($\cal_{\Gamma, r}$) is a special case of multicalibration where $\cF = \{1\}$, therefore is non-contextual. 
When $\Gamma$ is the mean (and thus $V(\gamma, y) = \gamma-y$), \eqref{eq:mcal_def_intro} and \eqref{eq:smcal_def_intro} clearly recover \eqref{eq:mcal} and \eqref{eq:swap_multical_comparator}.
}

\subsection{Related Work}
Although there are several algorithms that achieve (mean) multicalibration in both the offline \citep{hebert2018multicalibration, jung2021moment,  gopalan_et_al:LIPIcs.ITCS.2022.79, gopalan_et_al:LIPIcs.ITCS.2022.79, gupta2022online, globus2023multicalibration, gopalan2023swap, deng_et_al:LIPIcs.ITCS.2023.41} and online settings \citep{gupta2022online, haghtalab2023unifying, garg2024oracle, luo2025improved, ghugeimproved}, the resulting offline sample complexities/online multicalibration errors are either quite worse compared to the calibration counterparts, or do not generalize beyond simple classes of functions, e.g., linear functions. {For the ease of comparison, in this subsection, we review the relevant results for $\ell_{2}$-(swap) multicalibration and defer discussion for $\ell_{1}$-(swap) multicalibration to \pref{subsec:contributions} (they are related via the Cauchy-Schwartz inequality as $\smcal_{\Gamma, 2}(\cF) \le \smcal_{\Gamma, 1}(\cF) \le \sqrt{T \cdot \smcal_{\Gamma, 2}(\cF)}$ and $\mcal_{\Gamma, 2}(\cF) \le \mcal_{\Gamma, 1}(\cF) \le \sqrt{T \cdot \mcal_{\Gamma, 2}(\cF)}$)}. Particularly, in the online setting:
\begin{itemize}[leftmargin=*]
    \item \cite{garg2024oracle} proposed an oracle-efficient algorithm that achieves $\smcal_{\text{mean}, 2}(\cF) = \tilde{\cO}(T^{\frac{3}{4}})$ for any hypothesis class $\cF$ with bounded online complexity, and raised the open problem of whether it is possible to devise an oracle-efficient algorithm that achieves $\mcal_{\text{mean}, 2}(\cF) = \tilde{\cO}(\sqrt{T})$. In a restricted setting where the hypothesis class $\cF$ is finite, \cite{garg2024oracle} also developed an inefficient algorithm with running time proportional to $\abs{\cF}$ that achieves $\mcal_{\text{mean}, 2}(\cF) = \tilde{\cO}(\sqrt{T})$.
    
    \item Subsequent work by \cite{luo2025improved}  partially answered the question raised by \cite{garg2024oracle} in the affirmative by proposing an efficient algorithm that achieves $\smcal_{\text{mean}, 2}(\cF) = \tilde{\cO}(T^{\frac{1}{3}})$ when $\cF$ is the class of bounded linear functions. 
\end{itemize}  
Beyond mean multicalibration: 
\begin{itemize}[leftmargin=*]
    \item {For an elicitable property $\Gamma$, \cite{noarov2023statistical} proposed an inefficient algorithm that achieves $\mcal_{\Gamma, 2}(\cF) = \tilde{\cO}(\sqrt{T})$.\footnote{\cite{noarov2023statistical} use a slightly different definition of multicalibration compared to the one considered in this paper. However, their analysis extends gracefully to the definition of $\mcal_{\Gamma, 2}$ considered in this paper.}} 

    \item {In the special case that $\Gamma$ is the $q$-th quantile of a distribution, \cite{garg2024oracle} proposed an {oracle} efficient algorithm that achieves $\smcal_{\Gamma, 2}(\cF) = \tilde{\cO}(T^{\frac{3}{4}})$ for any hypothesis class $\cF$ with bounded online complexity, where as it is possible to achieve $\mcal_{\Gamma, 2} = \tilde{\cO}(\sqrt{T})$ via an inefficient algorithm \cite{bastani2022practical, roth2022uncertain, noarov2023statistical, garg2024oracle}.}
\end{itemize}
We summarize the above results in \pref{tab:comparison}. 

\begin{table}[h!]
\centering
\definecolor{lightblue}{RGB}{173, 216, 230}

\begin{tabularx}{\textwidth}{c X c c}
\toprule
Property ($\Gamma$) & Hypothesis class ($\cF$) & {$\mcal_{\Gamma, 2}(\cF)$} & \textbf{Efficiency} \\ 
\midrule
\multirow{3}{*}{Mean} 
 & Arbitrary with bounded complexity& \cellcolor{lightblue}$\tilde{\cO}(T^{\frac{3}{4}})$ \cite{garg2024oracle} & Oracle-efficient \\ 
\cmidrule(lr){2-4}
 & Finite & $\tilde{\cO}(\sqrt{T})$ \cite{garg2024oracle} & Inefficient \\ 
\cmidrule(lr){2-4}
 & Linear functions & \cellcolor{lightblue}$\tilde{\cO}(T^{\frac{1}{3}})$ \cite{luo2025improved} & Efficient \\ 
\midrule
\multirow{2}{*}{$q$-th Quantile} 
 & Arbitrary with bounded complexity & \cellcolor{lightblue} $\tilde{\cO}(T^{\frac{3}{4}})$   \cite{garg2024oracle} & Oracle-efficient \\ 
\cmidrule(lr){2-4}
 & Finite & $\tilde{\cO}(\sqrt{T})$ \cite{roth2022uncertain} & Inefficient \\ 
\midrule
Elicitable $\Gamma$ 
 & Group membership functions & $\tilde{\cO}(\sqrt{T})$ \cite{noarov2023statistical} & Inefficient \\ 
\bottomrule
\end{tabularx}
\caption{Previously established multicalibration error bounds. Blue cells indicate that the achieved rates were derived for swap multicalibration, therefore, also hold for multicalibration.}
\label{tab:comparison}
\end{table}

When compared against the achievable rates for mean multicalibration, the above results raise the following immediate question: 

\begin{center}
    {\emph{Is it possible to devise an oracle-efficient algorithm which achieves improved (swap) multicalibration guarantees compared to \pref{tab:comparison} for an arbitrary elicitable property $\Gamma$
    and a general hypothesis class $\cF$?}}
\end{center}

\subsection{Contributions}\label{subsec:contributions}
{In this paper, we answer the question above in a strongly affirmative sense by proposing an oracle-efficient algorithm which meets all the above requirements simultaneously: for an elicitable property $\Gamma$ with a $\rho$-Lipschitz identification function,\footnote{This is a standard assumption invoked in several prior works, e.g., \citep{noarov2023statistical} invoke it for multicalibrating an elicitable property $\Gamma$, where as \citep{jungbatch, garg2024oracle} invoke it for multicalibrating $q$-th quantiles.} fixed $r \ge 2$,\footnote{{Although we also have an explicit bound on $\smcal_{\Gamma, r}(\cF)$ for $r \in [1, 2)$, for the sake of brevity, we refer to the bound for $r \ge 2$ throughout, {since a bound on $\smcal_{r}$ for $r \in [1, 2)$ can be obtained from a bound on $\smcal_2$ by applying the H\"older's inequality; see \pref{thm:smcal_general_result_introduction}.
}}
}
and $\cF$ with bounded sequential Rademacher complexity, \pref{alg:hp_algorithm_efficient} achieves $\smcal_{\Gamma, r}(\cF) = \tilde{\cO}(T^{\frac{1}{r + 1}})$ with high probability. 
{Notably, for $r = 2$, this improves the dependence of $T$ significantly over all the rows in \pref{tab:comparison} except for the work of \cite{luo2025improved}, while improving on \cite{luo2025improved} by generalizing to $r > 2$, a broader hypothesis class $\cF$, and an elicitable property $\Gamma$.} Towards achieving this result, we first propose an inefficient algorithm (\pref{alg:hp_algorithm}) that achieves $$\smcal_{\Gamma, r}(\cF) = \tilde{\cO}\bigc{{\rho ^ {r} T ^ {\frac{1}{r + 1}} + T^{\frac{1}{r + 1}} \bigc{\log \frac{\abs{\cF}}{\delta}} ^ {\frac{r}{2}}}}$$ with probability at least $1 - \delta$ for a finite class $\cF$. 
Subsequently, we propose our oracle-efficient algorithm by invoking an online agnostic learner in \pref{alg:hp_algorithm}.}

\begin{restatable}[Online Agnostic Learning \citep{ben2009agnostic, beygelzimer2015optimal}]{definition}{DefOnlineAgnostic}\label{def:online_weak_agnostic_learner_intro}
    Consider the following interaction between an online agnostic learner ($\owal$) and adversary for $n$ rounds: at each time $t \in [n]$, (a) the adversary reveals a context $x_t \in \cX$; (b) $\owal$ responds with a prediction $q_{t}(x_{t})$, where $q_{t}: \cX \to [-1, 1]$; (c) adversary reveals the outcome $\kappa_{t} \in [-1, 1]$. In online agnostic learning, the goal of $\owal$ is to output a sequence of test functions $q_{1}, \dots, q_{T}$ 
    whose cumulative correlation $\sum_{t = 1} ^ {T} q_{t}(x_{t}) \kappa_{t}$ with the outcome sequence $\kappa_{1}, \dots, \kappa_{T}$ is comparable with the best hypothesis in a given hypothesis class $\cF$, i.e., \begin{align}\label{eq:WAL_regret}
        \sup_{f \in \cF} \sum_{t = 1} ^ {n} f(x_{t}) \kappa_{t} \le \sum_{t = 1} ^ {n} q_{t}(x_t) \kappa_{t} + \mathsf{Reg}(\cF, n),
    \end{align}
where $\mathsf{Reg}(\cF, n) > 0$ denotes the regret incurred by $\owal$.
\end{restatable}

 \pref{alg:hp_algorithm} is inefficient since it requires enumerating over each $f \in \cF$, whereas \pref{alg:hp_algorithm_efficient} achieves oracle-efficiency by instantiating $2N$ copies of $\owal$, parameterized by $(i, \sigma) \in [N] \times \{\pm 1\}$, and auditing against the auxiliary test function $q_{t, i, \sigma}$ provided by $\owal_{i, \sigma}$ for each $(i, \sigma)$ instead. In the following result, which is a special case of \pref{thm:smcal_general_result} (the main result of this paper), we derive a bound on $\smcal_{\Gamma, r}(\cF)$ by assuming the existence of an $\owal$ with a favorable regret guarantee. 
 \begin{theorem}\label{thm:smcal_general_result_introduction}
    {Fix a $r \ge 1$ and an elicitable property $\Gamma$ with a $\rho$-Lipschitz identification function, and assume that there exists an $\owal$ for which
    $\mathsf{Reg}(\cF, n) = \tilde{\cO}(\sqrt{n}\comp(\cF))$, where $\comp(\cF)$
    is a complexity measure of $\cF$ that is independent of $n$. Then, for $r \ge 2$, \pref{alg:hp_algorithm_efficient} achieves \begin{align*} 
        \smcal_{\Gamma, r}(\cF) = \tilde{\cO}\bigc{\bigc{\rho ^ {r} + \bigc{\log \frac{1}{\delta}} ^ {\frac{r}{2}} +  \comp(\cF) ^ {r}} T ^ {\frac{1}{r + 1}}},
    \end{align*}
    with probability at least $1 - \delta$. Consequently, for $r \in [1, 2)$, \pref{alg:hp_algorithm_efficient} achieves 
    \begin{align*}
        \smcal_{\Gamma, r}(\cF) = \tilde{\cO}\bigc{\bigc{\rho ^ {r}  +  \bigc{\log \frac{1}{\delta}} ^ {\frac{r}{2}} +  \comp(\cF) ^ {r}} T ^ {1 - \frac{r}{3}}}
    \end{align*}
    with probability at least $1 - \delta$.}\footnote{{Our bounds are achieved without the knowledge of $\rho$, $\mathsf{Reg}(\cF, n)$, and $\delta$. We did not make an attempt to optimize our bounds when these parameters are known.}}
\end{theorem}
To instantiate the above result: 
    \begin{itemize}[leftmargin=*]
        \item When $\cF$ is the class of linear functions, the Online Gradient Descent (OGD) algorithm \citep{zinkevich2003online} achieves {$\mathsf{Reg}(\cF, n) = {\cO}(\sqrt{n})$}, therefore, $\smcal_{\Gamma, 2}(\cF) = \tilde{\cO}(T^{\frac{1}{3}})$. This not only recovers the result of \cite{luo2025improved} for $\smcal_{\text{mean}, 2}(\cF)$ but also generalizes it for any $r > 2$ and an arbitrary elicitable property $\Gamma$.
        \item When $\cF$ is a finite class, the {Multiplcative Weights Update (MWU) algorithm \citep{freund1997decision} 
        achieves $\mathsf{Reg}(\cF, n) = {\cO}(\sqrt{n \log \abs{\cF}})$},
        thereby implying that $\smcal_{\Gamma, r}(\cF) = \tilde{\cO}(T^{\frac{1}{r + 1}}{(\log \abs{\cF})^{\frac{r}{2}}})$. 
        
        \item Similar to \cite{okoroafor2025near}, via standard learning-theoretic arguments and results from \cite{rakhlin2010online}, we show that there exists an $\owal$ that achieves {$\mathsf{Reg}(\cF, n) = \cO(n \cdot \Re^{\text{seq}}(\cF, n))$, where $\Re^{\text{seq}}(\cF, n)$ is the sequential Rademacher complexity of $\cF$ over $n$ rounds of interaction between an algorithm and adversary}. Therefore, whenever $\Re^{\text{seq}}(\cF, n) = \tilde{\cO}\bigc{\frac{\comp(\cF)}{\sqrt{n}}}$, we obtain {$\mathsf{Reg}(\cF, n) = \tilde{\cO}(\sqrt{n} \comp(\cF))$} and $\smcal_{\Gamma, r}(\cF) = \tilde{\cO}(T^{\frac{1}{r + 1}})$. Both the linear class and finite $\cF$ satisfy $\Re^{\text{seq}}(\cF, n) = \tilde{\cO}\bigc{\frac{\comp(\cF)}{\sqrt{n}}}$ with $\comp(\cF) = \cO(1)$ for the linear class and $\comp(\cF) = \log \abs{\cF}$ for the finite class, however, importantly, they admit an explicit algorithm that achieves {$\mathsf{Reg}(\cF, n) = \tilde{\cO}(\sqrt{n} \comp(\cF))$}. For hypothesis classes (beyond the above two classes) for which $\Re^{\text{seq}}(\cF, n) = \tilde{\cO}\bigc{\frac{\comp(\cF)}{\sqrt{n}}}$, we refer to \cite{rakhlin2014statistical}. 
    \end{itemize} 
    When instantiated for means, the above set of results not only completely resolves the open question raised in \cite{garg2024oracle} on the possibility of an oracle-efficient algorithm that achieves $\mcal_{\text{mean}, 2}(\cF) = \tilde{\cO}(\sqrt{T})$, but also shows that it is possible to achieve a much stronger guarantee $\smcal_{\text{mean}, 2}(\cF) = \tilde{\cO}(T^{\frac{1}{3}})$. More generally, for an elicitable property $\Gamma$, \pref{alg:hp_algorithm_efficient} achieves $\smcal_{\Gamma, 2} = \tilde{\cO}(T^{\frac{1}{3}})$, thereby improving upon several previous works \cite{noarov2023statistical, garg2024oracle, ghugeimproved, luo2025improved}. 

\begin{remark}\label{rem:ell1_vs_ell2}
    {For $\ell_{1}$-multicalibration, a recent work \cite{ghugeimproved} proposed an inefficient algorithm that achieves $\mcal_{\text{mean}, 1}(\cF) = \tilde{\cO}(T^{\frac{2}{3}})$ for a finite class $\cF$ (see also \cite{noarovhigh} for a different approach to obtain the same result) and an oracle-efficient algorithm that achieves $\mcal_{\text{mean}, 1}(\cF) = \tilde{\cO}(T^{\frac{3}{4}})$ for a general hypothesis class $\cF$. Since $\smcal_{\Gamma, 1}(\cF) \le \sqrt{T \cdot \smcal_{\Gamma, 2}(\cF)}$, we achieve $\smcal_{\text{mean}, 1}(\cF) = \tilde{\cO}(T^{\frac{2}{3}})$ in an oracle-efficient manner, thereby also significantly improving upon \cite{noarovhigh, ghugeimproved}.}
\end{remark}
\section{Technical Overview} 
We provide a technical overview of our work in this section.
\subsection{Warm up with $\cal_{\text{mean, 2}}$: A new general approach} \label{subsec:warm_up}
 We begin by considering the simpler problem of online $\ell_2$-calibration {for means}. Prior works for calibrating means achieve $\cal_{\text{mean}, 2} = \tilde{\cO}(T^{\frac{1}{3}})$ by either using the equivalence between $\ell_2$-calibration and swap regret of the squared loss \citep{luo2025simultaneous, fishelsonfull} or via the concept of calibeating \citep{foster2023calibeating}. However, in \pref{lem:swap_impossible_inequality} in the appendix, we show that when calibrating $q$-th quantiles, the $\ell_{2}$-quantile calibration error cannot be bounded by any strictly increasing and invertible function {$\zeta$} of the swap regret of any loss function, thereby rendering the swap regret-based approach futile. Furthermore, the calibeating-based approach also does not readily generalize beyond means. 
    
     {To circumvent this, we discover yet another approach to achieve $\cal_{\text{mean}, 2} = \tilde{\cO}(T^{\frac{1}{3}})$, which, unlike the previous approaches, can be directly generalized to properties beyond means. 
     Specifically, we observe that to achieve $\cal_{\text{mean}, 2} = \tilde{\cO}(T^{\frac{1}{3}})$, it suffices to      
     derive a high probability bound of the form \begin{align}\label{eq:high_probability_bound_needed}
         \abs{\sum_{t = 1} ^ {T} \ind{p_t = z_i} (y_t - p_t)} = \tilde{\cO}\bigc{\sqrt{n_i} + \frac{n_i}{N}},
     \end{align} 
     where we have ignored dependence on the failure probability $\delta$ for simplicity. 
     Before explaining why this is sufficient, we first point out that 
     \eqref{eq:high_probability_bound_needed} is a natural goal due to the following interpretation: if the forecaster could witness the conditional distribution of $y_{t}$ (e.g., if the adversary moves first instead and at every time $t \in [T]$, it reveals the conditional distribution of $y_t$),  
     then a simple truthful strategy which computes $\tilde{p}_t$ as the conditional expectation of $y_t$  
     and predicts $p_{t} = \argmin_{z \in \cZ} \abs{z - \tilde{p}_t}$ by rounding $\tilde{p}_t$ to the nearest point in $\cZ$ achieves \eqref{eq:high_probability_bound_needed}. This is because, such a forecasting strategy has to pay for the rounding error $\frac{1}{N}$ per round for $n_i$ rounds, and a variance term $\tilde{\cO}(\sqrt{n_i})$ due to not predicting the true label but rather its expectation.
     Of course, since the forecaster has to decide $p_{t}$ without any knowledge about $y_{t}$, it cannot implement this strategy, however, the discussion points towards the possibility that \eqref{eq:high_probability_bound_needed} might be achievable.}
     
     Notably, \eqref{eq:high_probability_bound_needed} indeed implies that $\cal_{\text{mean}, 2} = \tilde{\cO}(T^{\frac{1}{3}})$ with high probability. This is because, $$n_{i} \bigc{\frac{1}{n_i} \sum_{t = 1} ^ {T} \ind{p_t = z_i} (y_t - p_t)}^{2} = \tilde{\cO}\bigc{1 + \frac{n_{i}}{N ^ {2}}},$$ which when summed over $i$ yields $\cal_{\text{mean}, 2} = \sum_{i \in [N]} n_{i} \bigc{\frac{1}{n_i} \sum_{t = 1} ^ {T} \ind{p_t = z_i} (y_t - p_t)}^{2} = \tilde{\cO}\bigc{N + \frac{T}{N ^ {2}}} = \tilde{\cO}(T^{\frac{1}{3}})$ on choosing the discretization size $N = \Theta(T^{\frac{1}{3}})$. More generally, for $r > 2$, we obtain $$n_{i} \abs{\frac{1}{n_i} \sum_{t = 1} ^ {T} \ind{p_t = z_i} (y_t - p_t)} ^ {r} = \tilde{\cO}\bigc{n_{i} ^ {1 - r/2} + \frac{n_{i}}{N ^ {r}}} = \tilde{\cO}\bigc{1 + \frac{n_{i}}{N ^ {r}}},$$
     and thus $\cal_{\text{mean}, r} = \tilde{\cO}\bigc{N + \frac{T}{N ^ {r}}} = \Tilde{\Theta}(T^{\frac{1}{r + 1}})$ on choosing $N = \Theta(T^{\frac{1}{r + 1}})$.
     {Furthermore, this approach is readily generalizable to an arbitrary elicitable property $\Gamma$ with an identification function $V$, by simply replacing $y_t - p_t$ in~\eqref{eq:high_probability_bound_needed} with $V(p_t, y_t)$}.

     So how can one achieve~\eqref{eq:high_probability_bound_needed}? 
     In fact, a recent work by \cite{hu2024predict} almost achieves \eqref{eq:high_probability_bound_needed}
    in their algorithm designed for a different problem of simultaneous swap regret minimization. The only caveat is that their algorithm only ensures an in expectation bound on $\abs{\sum_{t = 1} ^ {T} \ind{p_t = z_i} (y_t - p_t)}$,
     which, as one can verify, is not sufficient to achieve $\cal_{\text{mean}, 2} = \tilde{\cO}(T^{\frac{1}{3}})$ or even $\mathbb{E}[\cal_{\text{mean}, 2}] = \tilde{\cO}(T^{\frac{1}{3}})$. Since the algorithm of \cite{hu2024predict} uses the expert algorithm McMwC of~\citep{chen2021impossible} as a subroutine, a natural approach to achieve the desired high probability bound would be to obtain a high probability external regret guarantee for McMwC, which is unfortunately mentioned as an open problem in \cite{chen2021impossible} and is still unsolved to the best of our knowledge.
     
     {Despite this fact, it turns out that \eqref{eq:high_probability_bound_needed} can still be achieved in a simple manner that is agnostic to the high probability external regret guarantee of McMwC, which is quite surprising.} Towards achieving this, we propose two simplifications to the algorithm of \cite{hu2024predict}, {the first of which is critical for achieving \eqref{eq:high_probability_bound_needed} and the other is important for 
     subsequent developments in the contextual setting (\pref{subsec:ineff_generalization}). We start with a brief introduction of the algorithm of \cite{hu2024predict}.} For each $i \in [N - 1]$, let $\cI_{i} = \left[\frac{i - 1}{N}, \frac{i}{N}\right)$ and $\cI_{N} = \left[\frac{N - 1}{N}, 1\right]$, so that $\cI_{1}, \dots, \cI_{N}$ represents a partition of $[0, 1]$.
    Informally, their algorithm defines a $2N$-expert problem corresponding to a careful choice of the gain function $\ell_{i, \sigma}(p, y) \coloneqq \sigma \ind{p \in \cI_i} (p - y)$ for each expert $(i, \sigma) \in [N] \times \{\pm 1\}$. At each time $t \in [T]$, the algorithm obtains a probability distribution $\{w_{t, i, \sigma}\}$ over the experts via an expert problem subroutine $\mathsf{ALG}$, which is instantiated to be the McMwC algorithm. Subsequently, it hedges against the experts by defining a function $h_{t}(P)$ that depends on the weights $\{w_{t, i, \sigma}\}$ and the gains $\{\ell_{i, \sigma}\}$ as (with $\mathscr{F}_{t - 1}$ being the filtration generated by the random variables $p_{1}, \dots, p_{t - 1}, y_{1}, \dots, y_{t - 1}$)
    \begin{align*}
        h_{t}(P) \coloneqq \max_{y \in \{0, 1\}} \mathbb{E}_{p \sim P}\bigs{\sum_{(i, \sigma)} w_{t, i, \sigma} \ell_{i, \sigma}(p, y) | \mathscr{F}_{t - 1}},
    \end{align*} obtains a distribution $P_{t}$ satisfying $h_{t}(P_{t}) = \cO\bigc{\frac{1}{T}}$, samples $\tilde{p}_t \sim P_t$, and forecasts $p_{t} = z_{i_t}$, where $i_{t} \in [N]$ is such that $\tilde{p}_{t} \in \cI_{i_t}$. {Finally, after observing $y_t$, we feed the gain $\ell_{t, i, \sigma} \coloneqq \ell_{i, \sigma}(\tilde{p}_t, y_t)$ for each $(i, \sigma)$ to $\alg$. Our proposed modifications are as follows:}
    \begin{enumerate}[leftmargin=*]
        \item \label{item:modification_I} Instead of feeding the loss $\ell_{i, \sigma}(\tilde{p}_t, y_{t})$, we feed a modified loss $\ell_{t, i, \sigma} = \mathbb{E}_{p \sim P_{t}} [\ell_{i, \sigma}(p, y_t) | \mathscr{F}_{t - 1}]$ for each $(i, \sigma)$.
        \item \label{item:modification_II} While \cite{hu2024predict} argue that {a distribution $P_t$ satisfying $h_{t}(P_t) = \cO\bigc{\frac{1}{T}}$} can be obtained in $\text{poly}(T)$ time by using the FTPL-based approach suggested by \cite{noarovhigh}, we provide a simple closed-form construction of $P_{t}$ that is supported on at most two points, {following a similar idea utilized in \cite{gupta2022online, okoroafor2025near}}. To state our construction, we define the function $\Phi_{t}: [0, 1] \to [-1, 1]$ as ${\Phi_{t}(p) \coloneqq \sum_{(i, \sigma)} w_{t, i, \sigma} \sigma \ind{p \in \cI_{i}}}$. If $\Phi_{t}(0) > 0$, we choose $P_t$ that is only supported on $0$; else if $\Phi_{t}(1) \le 0$, we choose $P_t$ that is only supported on $1$; else we choose $i \in \{0, \dots, T - 1\}$ such that $\Phi_{t}\bigc{\frac{i}{T}} \Phi_{t}\bigc{\frac{i + 1}{T}} \le 0$ and $P_t$ such that \begin{align*}
        P_t\bigc{\frac{i}{T}} = \frac{\abs{\Phi_{t}\bigc{\frac{i + 1}{T}}}}{\abs{\Phi_{t}\bigc{\frac{i}{T}}} + \abs{\Phi_{t}\bigc{\frac{i + 1}{T}}}}, \quad P_t\bigc{\frac{i + 1}{T}} = \frac{\abs{\Phi_{t}\bigc{\frac{i}{T}}}}{\abs{\Phi_{t}\bigc{\frac{i}{T}}} + \abs{\Phi_{t}\bigc{\frac{i + 1}{T}}}}.
\end{align*}
    \end{enumerate}
    {To see why the first modification leads to \eqref{eq:high_probability_bound_needed}, we realize the following: \begin{align*}
         \abs{\sum_{t = 1} ^ {T} \ind{\tilde{p}_t \in \cI_i} (y_t - p_{t})} &\le \abs{\sum_{t = 1} ^ {T} \ind{\tilde{p}_t \in \cI_i} (y_t - \tilde{p}_t)} + \abs{\sum_{t = 1} ^ {T} \ind{\tilde{p}_t \in \cI_i} (\tilde{p}_t - p_t)} \\ 
         &\le \max_{\sigma \in \{\pm 1\}} \sum_{t = 1} ^ {T} \ell_{i, \sigma}(\tilde{p}_t, y_t) + \frac{n_i}{N} \\
         &\le \max_{\sigma \in \{\pm 1\}} \sum_{t = 1} ^ {T} \ell_{t, i, \sigma} + \max_{\sigma \in \{\pm 1\}} \sum_{t = 1} ^ {T} \ell_{i, \sigma}(\tilde{p}_t, y_t) - \ell_{t, i, \sigma} + \frac{n_i}{N},
    \end{align*}
    where the second inequality follows from the definition of $\ell_{i, \sigma}$ and $\abs{\tilde{p}_t - p_{t}} \le \frac{1}{N}$. Since the McMwC algorithm receives gain $\ell_{t, i, \sigma}$ for each $(i, \sigma)$, it follows from the regret guarantee of McMwC that (with $\mathbb{P}_{t}[.] \coloneqq \mathbb{P}(. | \mathscr{F}_{t - 1})$) \begin{align*}
        \sum_{t = 1} ^ {T} \ell_{t, i, \sigma} \le \sum_{t = 1} ^ {T} \sum_{(i', \sigma')} w_{t, i', \sigma'} \ell_{t, i', \sigma'} + \tilde{\cO}\bigc{\sqrt{\sum_{t = 1} ^ {T} \ell_{t, i, \sigma} ^ {2}}} = \tilde{\cO}\bigc{\sqrt{\sum_{t = 1} ^ {T} \mathbb{P}_t (\tilde{p}_t \in \cI_i)}},
    \end{align*}
    where the equality follows since for each $t$, $\sum_{(i', \sigma')} w_{t, i', \sigma'} \ell_{t, i', \sigma'} \le h_t(P_t) = \O(\frac{1}{T})$ and $\ell_{t, i, \sigma} ^ 2 \le \mathbb{P}_t (\tilde{p}_t \in \cI_i)$ by the definition of $\ell_{t, i, \sigma}$. Next, observe that the sequence $X_{t} \coloneqq \ell_{i, \sigma}(\tilde{p}_t, y_t) - \ell_{t, i, \sigma} = \ell_{i, \sigma}(\tilde{p}_t, y_t) - \mathbb{E}_{p \sim P_{t}} [\ell_{i, \sigma}(p, y_t) | \mathscr{F}_{t - 1}]$ is a martingale difference sequence (conditioned on $\mathscr{F}_{t - 1}$, the random variables $\tilde{p}_{t}$ (or $p_t$) and $y_t$ are independent) with cumulative variance $\sum_{t = 1} ^ {T}\mathbb{E}_{t}\bigs{X_t ^ {2}} \le \sum_{t = 1} ^ {T} \mathbb{E}_t [\ell_{i, \sigma}(\tilde{p}_t, y_t) ^ {2}] \le \sum_{t = 1} ^ {T} \mathbb{P}_t(\tilde{p}_t \in \cI_i)$, where $\mathbb{E}_{t}[.] \coloneqq \mathbb{E}[. | \mathscr{F}_{t - 1}]$. Therefore, applying Freedman's inequality (\pref{lem:Freedman}), we obtain $\abs{\sum_{t = 1} ^ {T} X_{t}} = \tilde{\cO}\bigc{\sqrt{\sum_{t = 1} ^ {T} \mathbb{P}_t(\tilde{p}_t \in \cI_i)}}$ with high probability. Combining everything, \begin{align*}
        \abs{\sum_{t = 1} ^ {T} \ind{\tilde{p}_t \in \cI_i} (y_t - p_t)} = \tilde{\cO}\bigc{\sqrt{\sum_{t = 1} ^ {T} \mathbb{P}_t(\tilde{p}_t \in \cI_i)} + \frac{n_i}{N}} = \tilde{\cO}\bigc{\sqrt{n_i} + \frac{n_i}{N}},
    \end{align*}
    since $\sqrt{\sum_{t = 1} ^ {T} \mathbb{P}_t(\tilde{p}_t \in \cI_i)} = \tilde{\cO}(\sqrt{n}_i)$ with high probability by applying Freedman's inequality again to the sequence $Z_{t} \coloneqq \mathbb{P}_{t}(\tilde{p}_t \in \cI_i) - \ind{\tilde{p}_t \in \cI_i}$}.
    This thus achieves \eqref{eq:high_probability_bound_needed} without the need of getting a high-probability version of the regret guarantee of McMwC.
    
    {We remark again that while the second modification has not been used yet in the above discussion, it will subsequently be quite helpful in the contextual setting, where we utilize it to obtain tail bounds on the sum of a sequence that is not necessarily a martingale difference sequence. Moreover, we do not flesh out details of the above discussion in the paper since it is a special case of the contextual setting (\pref{subsec:inefficient_multicalibration}) and  \pref{lem:hp_deviation_inefficient} is a generalized version of \eqref{eq:high_probability_bound_needed}.}

    \subsection{An inefficient generalization to the contextual setting} \label{subsec:ineff_generalization} By combining the two modifications above, 
    we generalize to an arbitrary elicitable property and also to the contextual setting where the forecaster receives a context $x_t \in \cX$, predicts $p_t \in [0, 1]$, and finally observes $y_{t} \sim Y_{t}$, where $Y_{t}$ is a distribution over $\cY \subseteq [0, 1]$ that is chosen by the adversary.  In \pref{subsec:inefficient_multicalibration}, we propose an inefficient algorithm (\pref{alg:hp_algorithm}), which for a fixed $r \ge 2$, finite $\cF$, and an elicitable property $\Gamma$ {with an identification function $V$ so that the marginal $V(., Y) \coloneqq \mathbb{E}_{y \sim Y} [V(., y)]$ is
    $\rho$-Lipschitz (\pref{ass:assumption_identification})}, achieves $$\smcal_{\Gamma, r}(\cF) = \tilde{\cO}\bigc{{\rho ^ {r} T ^ {\frac{1}{r + 1}} + T^{\frac{1}{r + 1}} \bigc{\log \frac{\abs{\cF}}{\delta}} ^ {\frac{r}{2}}}}$$ with probability at least $1 - \delta$ (\pref{thm:smcal_bound_inefficient}). \pref{alg:hp_algorithm} defines a $2N\abs{\cF}$-expert problem corresponding to a choice of the gain function $\ell_{f, i, \sigma}(p, x, y) \coloneqq \sigma \ind{p \in \cI_i} f(x) V(p, y)$ for each $(f, i, \sigma) \in \cF \times [N] \times \{\pm 1\}$. As mentioned in \pref{subsec:warm_up} above, at each time $t \in [T]$, the algorithm hedges against the experts by defining the function $$h_{t}(P) \coloneqq \sup_{y \in \cY} \mathbb{E}_{p \sim P}\bigs{\sum_{(f, i, \sigma)} w_{t, f, i, \sigma} \ell_{f, i, \sigma}(p, x_t, y) | \mathscr{F}_{t - 1}},$$
    where $\{w_{t, f, i, \sigma}\}$ is the probability distribution over experts as output by the expert problem subroutine $\mathsf{ALG}$. 
   Subsequently, it obtains a distribution $P_{t}$ satisfying $h_t(P_t) = \cO\bigc{\frac{1}{T}}$ in a similar manner as \pref{subsec:warm_up},  samples $\tilde{p}_t \sim P_{t}$, and forecasts $p_{t} = \frac{i_{t}}{N}$, where $i_{t} \in [N]$ is the unique index such that $\tilde{p}_{t} \in \cI_{i_t}$. Finally, after observing $y_t$, we feed the gain ${\ell_{t, f, i, \sigma} \coloneqq \mathbb{E}_{p \sim P_t}\bigs{\ell_{f, i, \sigma}(p, x_t,  y_{t}) | \mathscr{F}_{t - 1}}}$ for each $(f, i, \sigma)$ to $\alg$.

    Although the proposed modifications
    and generalization to the algorithm of \cite{hu2024predict} are simple, their combination with a martingale-based analysis {using Freedman's inequality (\pref{lem:Freedman})} in the style of \cite{luo2025improved} results in a substantially different and technical proof compared to both \cite{hu2024predict} and \cite{luo2025improved}. Moreover, the central idea in \cite{luo2025improved, garg2024oracle} is a reduction of swap multicalibration to contextual swap regret, which is quite different from our approach. As a highlight of our proof, we utilize our explicit construction of $P_{t}$ to obtain tail bounds on $\abs{\sum_{t = 1} ^ {T} W_{t}}$ for a non-martingale sequence $\{W_{t}\}_{t = 1} ^ {T}$. To elaborate on this further, we provide a detailed proof sketch of \pref{thm:smcal_bound_inefficient} here. Similar to calibration for means (\pref{subsec:warm_up}), we begin by deriving a high probability bound on the quantity $\sup_{f \in \cF} \abs{\sum_{t = 1} ^ {T} \ind{p_t = z_i} f(x_t) V(p_t, y_t)}$ that holds for all $i \in [N]$ (\pref{lem:hp_deviation_inefficient}), in the following manner:
\begin{enumerate}[label=(\roman*), leftmargin=*]
    \item Fix $i \in [N], f \in \cF$, and failure probability $\delta$. Recall that $n_{i} = \sum_{t = 1} ^ {T} \ind{p_t = z_i} = \sum_{t = 1} ^ {T} \ind{\tilde{p}_t \in \cI_i}$ is the number of time instants the prediction made is $p_{t} = z_i$, or equivalently $\tilde{p}_t \in \cI_i$. It follows from the triangle inequality that $\abs{\sum_{t = 1} ^ {T} \ind{p_t = z_i} f(x_t) V(p_t, y_t)} \le \abs{\sum_{t = 1} ^ {T} X_t} + \abs{\sum_{t = 1} ^ {T} U_t} + \abs{\sum_{t = 1} ^ {T} Z_t} + \abs{\sum_{t = 1} ^ {T} W_t}$, where the sequences $\{X_{t}\}, \{U_{t}\}, \{Z_{t}\}$, and $\{W_{t}\}$ are defined as
    \begin{align*}
        X_{t} &\coloneqq \ind{\tilde{p}_t \in \cI_i} f(x_t) \cdot \bigc{V(z_i, y_t) - \mathbb{E}_{t} [V(z_i, y_{t})]}, \\
        U_{t} &\coloneqq \ind{\tilde{p}_t \in \cI_i} f(x_t) \cdot  (\mathbb{E}_t [V(z_i, y_t)] - \mathbb{E}_{t} [V(\tilde{p}_t, y_t)]), \\
        Z_{t} &\coloneqq \ind{\tilde{p}_t \in \cI_i} f(x_t) \cdot V(\tilde{p}_t, y_t), \\
        W_{t} &\coloneqq \ind{\tilde{p}_t \in \cI_i} f(x_t) \cdot (\mathbb{E}_{t}[V(\tilde{p}_t, y_t)] - V(\tilde{p}_t, y_t)).
    \end{align*}
    In the next steps, we discuss our approach to bound the sum of each individual sequence. As already mentioned before, conditioned on $\mathscr{F}_{t - 1}$, the random variables $p_{t}$ (or $\tilde{p}_{t}$) and $y_{t}$ are independent. Since $V(., Y_t)$ is $\rho$-Lipschitz and $\abs{p_{t} - \tilde{p}_{t}} \le \frac{1}{N}$, we immediately obtain $\abs{U_{t}} \le \frac{\rho}{N}$. 
    \item \label{step:II} Observe that $\{X_{t}\}$ is a martingale difference sequence with $\abs{X_t} \le 2$. Applying Freedman's inequality carefully over a dyadic partition of the interval $\bigs{\tilde{\Theta}\bigc{\frac{\log \frac{1}{\delta}}{\sqrt{T}}}, \frac{1}{2}}$, similar to \cite[Lemma 8]{luo2025improved}, we obtain \begin{align*}
        \abs{\sum_{t = 1} ^ {T} X_{t}} = \tilde{\cO}\bigc{\sqrt{\bigc{\sum_{t = 1} ^ {T} \mathbb{P}_{t}(\tilde{p}_t \in \cI_i)} \log \frac{1}{\delta}} + \log \frac{1}{\delta}}.
    \end{align*}
    \item Next, we bound $\abs{\sum_{t = 1} ^ {T} Z_{t}}$. It follows from the definition of the gain function $\ell_{f, i, \sigma}$ that \begin{align*}
    \abs{\sum_{t = 1} ^ {T} Z_{t}} &= \max_{\sigma \in \{\pm 1\}} \sum_{t = 1} ^ {T} \ell_{f, i, \sigma}(\tilde{p}_t, x_t, y_t) \\
    &\le \max_{\sigma \in \{\pm 1\}} \sum_{t = 1} ^ {T} \ell_{t, f, i, \sigma} + \max_{\sigma \in \{\pm 1\}} \sum_{t = 1} ^ {T} \ell_{f, i, \sigma} (\tilde{p}_t, x_t, y_t) - \mathbb{E}_{\tilde{p}_t}[\ell_{f, i, \sigma} (\tilde{p}_t, x_t, y_t) | \mathscr{F}_{t - 1}].
    \end{align*}
    Invoking the regret guarantee of $\alg$ (\pref{lem:expert_minimization_alg}), we obtain \begin{align*}
        \max_{\sigma \in \{\pm 1\}} \sum_{t = 1} ^ {T} \ell_{t, f, i, \sigma} \le \sum_{t = 1} ^ {T} \sum_{(f, i, \sigma)} w_{t, f, i, \sigma} \ell_{t, f, i, \sigma} + \tilde{\cO}\bigc{\log (\abs{\cF}N) + \sqrt{\bigc{\sum_{t = 1} ^ {T} \mathbb{P}_{t}(\tilde{p}_t \in \cI_i)} \log (\abs{\cF}N)}}.
    \end{align*}
    Moreover, $\sum_{t = 1} ^ {T} \sum_{(f, i, \sigma)} w_{t, f, i, \sigma} \ell_{t, f, i, \sigma} \le \sum_{t = 1} ^ {T} h_t (P_t) = \cO(1)$. For a fixed $\sigma \in \{\pm 1\}$, the sequence $\ell_{f, i, \sigma} (\tilde{p}_t, x_t, y_t) - \mathbb{E}_{\tilde{p}_t}[\ell_{f, i, \sigma} (\tilde{p}_t, x_t, y_t) | \mathscr{F}_{t - 1}]$ is a martingale difference sequence. Repeating a similar analysis as that done for the sequence $\{X_{t}\}$, we obtain \begin{align*}
        \abs{\sum_{t = 1} ^ {T} Z_{t}} = \tilde{\cO}\bigc{\sqrt{\bigc{\sum_{t = 1} ^ {T} \mathbb{P}_{t}(\tilde{p}_t \in \cI_i)} \log \frac{\abs{\cF} N}{\delta}} + \log \frac{\abs{\cF} N}{\delta}}.
    \end{align*}
    \item Towards bounding $\abs{\sum_{t = 1} ^ {T} W_{t}}$, we observe that the sequence $\{W_{t}\}$ is not necessarily a martingale difference sequence, due to the indicator $\ind{\tilde{p}_t \in \cI_i}$. However, using the explicit form of $P_{t}$ (thanks to our second modification to \cite{hu2024predict}) and $\rho$-Lipschitzness of $V(., Y_{t})$, we argue that $\abs{\mathbb{E}_t[W_{t}]} \le \frac{\rho}{T}$. Thereafter, we consider the sequence $\{W_{t} - \mathbb{E}_t[W_{t}]\}$, which is clearly a martingale difference sequence, and bound $\abs{\sum_{t = 1} ^ {T} W_{t} - \mathbb{E}_t[W_t]}$ in an exactly same manner as \pref{step:II}. Since $\abs{\sum_{t = 1} ^ {T} W_t} \le \abs{\sum_{t = 1} ^ T W_t - \mathbb{E}_t[W_t]} + \abs{\sum_{t = 1} ^ {T} \mathbb{E}_t[W_t]}$ and $\abs{\mathbb{E}_t[W_t]} \le \frac{\rho}{T}$, we obtain \begin{align*}
    \abs{\sum_{t = 1} ^ {T} W_{t}} =  \tilde{\cO}\bigc{\sqrt{\bigc{\sum_{t = 1} ^ T \mathbb{P}_{t}(\tilde{p}_t \in \cI_i)} \log \frac{1}{\delta}} + \log \frac{1}{\delta}}.
\end{align*}
    \item Finally, we relate $\sum_{t = 1} ^ {T} \mathbb{P}_{t}(\tilde{p}_t \in \cI_i)$ with $n_{i}$ by bounding $\abs{\sum_{t = 1} ^ {T} \mathbb{P}_{t}(\tilde{p}_t \in \cI_i) - \ind{\tilde{p}_t \in \cI_i}}$. To achieve so, we consider the sequence $\bigcurl{\mathbb{P}_t(\tilde{p}_t \in \cI_i) - \ind{\tilde{p}_t \in \cI_i}}$, which is again a martingale difference sequence. Proceeding similarly as \pref{step:II}, we obtain \begin{align*}
        \abs{\sum_{t = 1} ^ {T} \ind{\tilde{p}_t \in \cI_i} - \mathbb{P}_t (\tilde{p}_t \in \cI_i)} &= \tilde{\cO}\bigc{\sqrt{\bigc{\sum_{t = 1} ^ T \mathbb{P}_{t}(\tilde{p}_t \in \cI_i)} \log \frac{1}{\delta}} + \log \frac{1}{\delta}}.
    \end{align*}
    Let $\alpha \coloneqq \sum_{t = 1} ^ {T} \mathbb{P}_t (\tilde{p}_t \in \cI_i), \beta = \log \frac{1}{\delta}$, and $\gamma > 0$ be the constant (including logarithmic terms) hidden in the $\tilde{\cO}(.)$ notation. Then the equation in the display above can be expressed as $\abs{n_i - \alpha} \le \gamma (\sqrt{\alpha \beta} + \beta)$. Solving for $\alpha$, we obtain $\alpha \le \gamma \sqrt{\beta} + \sqrt{n_i} + \sqrt{\gamma \beta} = \tilde{\cO}\bigc{\sqrt{\log \frac{1}{\delta}} + \sqrt{n_i}}$. 
\end{enumerate}

Combining everything and taking a union bound over all $i \in [N], f \in \cF$, we obtain $$\sup_{f \in \cF} \abs{\sum_{t = 1} ^ {T} \ind{p_t = z_i} f(x_t) V(p_t, y_t)} = \tilde{\cO}\bigc{\frac{\rho n_i}{N} + \sqrt{n_i \log \frac{\abs{\cF} N}{\delta}} + \log \frac{\abs{\cF} N}{\delta}}$$ 
with probability at least $1 - \delta$. The result of \pref{thm:smcal_bound_inefficient} then follows by directly plugging the above bound in the definition of $\smcal_{\Gamma, r}(\cF)$ \eqref{eq:smcal_def_intro}.

\subsection{An oracle efficient algorithm} Finally, in \pref{subsec:oracle_efficient_multicalibration}, we propose an oracle-efficient algorithm (\pref{alg:hp_algorithm_efficient}) by reducing to online agnostic learning. Our oracle-efficient algorithm builds on \pref{alg:hp_algorithm} and achieves efficiency by using $\owal$, similar to \cite{okoroafor2025near}. 
    We remark that while the idea of making \pref{alg:hp_algorithm} oracle-efficient is adopted from \cite{okoroafor2025near}, its combination with other techniques proposed in this paper results in an overall different algorithm and a significantly different analysis compared to \cite{okoroafor2025near}. 

    As mentioned briefly in \pref{subsec:contributions}, instead of enumerating over each $f \in \cF$, \pref{alg:hp_algorithm_efficient} instantiates $2N$ copies of $\owal$, parametrized by $(i, \sigma)$, and audits against the auxiliary test function $q_{t, i, \sigma}$ output by $\owal_{i, \sigma}$, i.e., we consider a modified function $$h_{t}(P) \coloneqq \sup_{y \in \cY} \mathbb{E}_{p \sim P}\bigs{\sum_{(i, \sigma)} w_{t, i, \sigma} \ell_{t, i, \sigma}(p, y) | \mathscr{F}_{t - 1}},$$ where $\ell_{t, i, \sigma}(p, y) \coloneqq \sigma \ind{p \in \cI_i} q_{t, i, \sigma}(x_t) V(p, y)$. Thereafter, we obtain a distribution $P_{t}$ satisfying $h_t(P_t) = \cO(\frac{1}{T})$ and make a prediction $p_t$ in a similar way as before. {After observing $y_t$, we feed the gain $\ell_{t, i, \sigma} \coloneqq \mathbb{E}_{p \sim P_{t}}[\ell_{t, i, \sigma}(p, y_t) | \mathscr{F}_{t - 1}]$ for each $(i, \sigma)$ to $\alg$ and a feedback outcome $\sigma V(p_t, y_t)$ to $\owal_{i_{t}, \sigma}$ for each $\sigma \in \{\pm 1\}$. 
    
    Next, we provide a proof sketch of the swap multicalibration guarantee  (\pref{thm:smcal_general_result_introduction}) of \pref{alg:hp_algorithm_efficient}. Let $\kappa_{t, i, \sigma} \coloneqq \sigma \ind{p_{t} = z_i} V(p_t, y_t)$.
    Fix a $i \in [N]$ and failure probability $\delta$. We begin by realizing that \begin{align*}\sup_{f \in \cF} \abs{\sum_{t = 1} ^ {T} \ind{p_t = z_i} f(x_t) V(p_t, y_t)} &= \sup_{\sigma \in \{\pm 1\}, f \in \cF} \sum_{t = 1} ^ {T} \sigma \ind{p_t = z_i} f(x_t) V(p_t, y_t) \\
    &= \sup_{\sigma \in \{\pm 1\}, f \in \cF} \sum_{t = 1} ^ {T} \kappa_{t, i, \sigma} f(x_{t}) \\
    &\le \sup_{\sigma \in \{\pm 1\}} \sum_{t = 1} ^ {T} \kappa_{t, i, \sigma} q_{t, i, \sigma}(x_{t}) + {\mathsf{Reg}(\cF, n_i)},
\end{align*}
where the last inequality follows from \pref{def:online_weak_agnostic_learner_intro}.
For the next step, we bound $\sup_{\sigma \in \{\pm 1\}} \sum_{t = 1} ^ {T} \kappa_{t, i, \sigma} q_{t, i, \sigma}(x_{t})$ by proceeding similar to the proof sketch in \pref{subsec:ineff_generalization}, which eventually leads to
$$\sup_{f \in \cF} \abs{\sum_{t = 1} ^ {T} \ind{p_t = z_i} f(x_t) V(p_t, y_t)} = \tilde{\cO}\bigc{\frac{\rho n_i}{N} + \sqrt{n_i \log \frac{\abs{\cF} N}{\delta}} + \log \frac{\abs{\cF} N}{\delta} + {\reg(\cF, n_i)}}$$ for all $i \in [N]$ with probability at least $1 - \delta$. The proof of \pref{thm:smcal_general_result_introduction} then follows by plugging the above bound to the definition of $\smcal_{\Gamma, r}(\cF)$ \eqref{eq:smcal_def_intro}.
}
    
    We conclude this section with the following remarks. 
    \begin{remark}[Oracle comparison]
        Note that compared to \cite{garg2024oracle}, which assumes an online squared error regression oracle for $\cF$ and \cite{ghugeimproved}, which assumes an offline optimization (multicalibration evaluation) oracle, we assume an online agnostic learner for $\cF$, {which is incomparable to both oracles}.
        Since multicalibration is generally considered a stronger notion compared to omniprediction \citep{gopalan_et_al:LIPIcs.ITCS.2022.79} and recent breakthrough by \cite{okoroafor2025near} assumes an online agnostic learner for $\cF$ for achieving oracle-efficient online omniprediction, it is reasonable to assume the same for multicalibration. Even in the distributional setting, multicalibration auditing using a weak agnostic learner is quite standard (see, for example \citep{hebert2018multicalibration}).
    \end{remark}
    \begin{remark}[Lower Bounds]
        While our results imply an efficient algorithm that achieves $\cal_{\Gamma, 2} = \tilde{\cO}(T^{\frac{1}{3}})$ for an elicitable property $\Gamma$, the question of whether this is tight for every $\Gamma$ (in terms of a matching lower bound) is wide open and beyond the scope of this work. For means, it is known that $\tilde{\Omega}(T^{0.08}) = \cal_{\text{mean}, 2} = \tilde{\cO}(T^{\frac{1}{3}})$, where the lower bound is due to \cite{qiao2021stronger, dagan2025breaking}. Closing the gap seems to require significant theoretical advancements.
    \end{remark}

\section{Preliminaries}

\paragraph{Notation.} For any $m \in [N]$, $[m]$ denotes the index set $\{1, \dots, m\}$. Throughout the paper, we reserve calligraphic alphabets, e.g., $\cP, \cZ$ to denote sets, where as probability distributions over sets are represented by upper-case letters. For a set $\cA$, $\Delta(\cA)$ represents the simplex over $\cA$, i.e., the set of all probability distributions over $\cA$. We use $\mathbb{I}[.]$ to represent the indicator function, which evaluates to $1$ if the predicate is true and $0$ otherwise. For a set $\cS$, $\bar{\cS}$ represents the complement $\Omega\backslash \cS$ of $\cS$, where the sample set $\Omega$ shall be clear from the context. Finally, we use the notations $\tilde{\cO}, \tilde{\Omega}, \tilde{\Theta}$ to hide lower-order logarithmic terms in $T$.

\subsection{Property Elicitation}
 Let the label space be $\cY \subseteq [0, 1]$ and $\cP$ be a family of probability distributions supported over $\cY$. We view a property $\Gamma: \cP \to [0, 1]$ as a mapping from a distribution to a scalar value.
 A loss function $\ell: [0, 1] \times \cY \to \mathbb{R}$ takes as input a prediction $\gamma$ about the property value, a label $y$, and maps it to a real number. We say that a property $\Gamma$ is \emph{elicitable} if there exists a loss function that is minimized at the true property value. This is formalized in the following definition, which is central to our work.
\begin{definition}[Strictly consistent loss function, Property elicitation] Fix a property $\Gamma$, a family of probability distributions $\cP$. A loss function $\ell$ is strictly $\cP$-consistent if $\Gamma(P) \in \argmin_{\gamma \in [0, 1]} \ell(\gamma, P)$
for all $P \in \cP$, where $\ell(\gamma, P) \coloneqq \mathbb{E}_{y \sim P} [\ell(\gamma, y)]$. We also say that $\ell$ elicits $\Gamma$ 
and that $\Gamma$ is elicitable if there exists a strictly $\cP$-consistent loss function. 
\end{definition}
In other words, $\ell$ is strictly $\cP$-consistent if for each $P \in \cP$, the true property value $\Gamma(P)$ can be obtained by minimizing the expected loss of $\ell$ over samples drawn from $P$. We provide several examples of elicitable properties: 
\begin{itemize}
    \item The squared loss $\ell(\gamma, y) = \frac{1}{2}(\gamma - y) ^ {2}$ elicits the mean; more generally, $\ell(\gamma, y) = (\gamma - y^{k}) ^ {2}$ elicits the \emph{raw moment} $\mathbb{E}[y ^ {k}]$, where $k \in \mathbb{N}$. 
    \item The $\ell_{1}$-loss $\ell(\gamma, y) = \frac{1}{2}\abs{\gamma - y}$ elicits the median; more generally, the pinball loss $\ell(\gamma, y) = (\gamma - y)\bigc{(1 - q)  \ind{y \le \gamma} - q \ind{y > \gamma}}$ elicits the $q$-th quantile for $q \in [0, 1]$.
    \item For a discrete distribution with finite support, the $0-1$ loss $\ell(\gamma, y) = \ind{\gamma \neq y}$ elicits the \emph{mode}.
    \item {For a fixed $\tau \in (0, 1)$, the loss $\ell(\gamma, y) = \frac{1}{2}(y - \gamma) ^ {2} \abs{\tau - \ind{\gamma - y \ge 0}}$ elicits the \emph{expectile}, 
    which corresponds to a $t$ that satisfies $\int_{-t} ^ {\infty} \abs{t - z} d\mu(z) = \tau \int_{-\infty} ^ {\infty} \abs{t - z} d\mu(z)$, where $\mu$ is the cumulative distribution function of $y$}.
    \item {Let $P$ be a distribution supported over $\{0, 1\}$ and $\Gamma(P) = \mathbb{P}(y = 1)$. A loss $\ell: [0, 1] \times \{0, 1\} \to \mathbb{R}$ is called \emph{proper} if $\mathbb{E}_{y \sim \gamma} [\ell(\gamma, y)] \le \mathbb{E}_{y \sim \gamma} [\ell(\gamma', y)]$ for all $\gamma, \gamma' \in [0, 1]$, e.g., squared loss, log loss $\ell(\gamma, y) = -y \log \gamma - (1 - y) \log (1 - \gamma)$, spherical loss $\ell(\gamma, y) = -\frac{\gamma y + (1 - \gamma) (1 - y)}{\sqrt{\gamma ^ 2 + (1 - \gamma) ^ {2}}}$, Tsallis entropy $\ell(\gamma, y) = (\alpha - 1) \gamma ^ {\alpha} - \alpha \gamma ^ {\alpha - 1} y$ for $\alpha > 1$, etc. More generally, given any concave function $f(\gamma)$, the loss $\ell(\gamma, y) = f(\gamma) + \ip{\partial f(\gamma)}{y - \gamma}$ is proper \cite{gneiting2007strictly}, where $\partial f$ represents a supergradient of $f$. Clearly, any proper loss $\ell$ elicits $\Gamma$ by definition.}
\end{itemize} 
Next, we define the notion of an identification function for a property, which intuitively measures the overestimate/underestimate of the true property value.

\begin{definition}[Identification function] Fix a property $\Gamma$ and a space of probability distributions $\cP$. A function $V: [0, 1] \times \cY \to \mathbb{R}$ is called a $\cP$-identification function for $\Gamma$ if $V(\gamma, P) = 0 \iff \gamma = \Gamma(P)$ for all $P \in \cP$, where $V(\gamma, P) \coloneqq \mathbb{E}_{y \sim P}[V(\gamma, y)]$. We say that $\Gamma$ is identifiable if there exists a $\cP$-identification function for $\Gamma$.

\end{definition}

As can be concluded from the definition, for any distribution $P \in \cP$, the true property value $\Gamma(P)$ is a root of the expected identification function over samples drawn from $P$. For example, $V(\gamma, y) = \gamma - y$ is an identification function for mean, whereas $V(\gamma, y) = \ind{y \le \gamma} - 0.5$ is an identification function for median; more generally, $V(\gamma, y) = \ind{y \le \gamma} - q$ is an identification function for the $q$-th quantile.  

A seminal result due to \cite{steinwart2014elicitation} is that under mild technical assumptions on the family $\cP$ and the property $\Gamma$, the following conditions are equivalent: (a) $\Gamma$ is elicitable; (b) there exists a strictly consistent loss function for $\Gamma$; (c) there exists an identification function for $\Gamma$; (d) the level sets $\Gamma_{\alpha} \coloneqq \{P \in \cP; \Gamma(P) = \alpha\}$ of $\Gamma$ are convex. An immediate consequence of the characterization is that it allows us to identify several properties that are not elicitable. For example, consider the conditional value at risk (also known as expected shortfall), which for a target quantile $q \in [0, 1]$ is defined as $\text{CVaR}_{q}(P) \coloneqq \mathbb{E}_{y \sim P}[y | y > f_{q}(P)]$, where $f_q(P)$ represents the $q$-th quantile of $P$. CVaR is of central significance in financial risk assesment, however, is not elicitable since its level sets are not convex \cite{gneiting2011making}. Similarly, when $y$ is a binary random variable and $\cP$ is the set of all Bernoulli distributions with mean $p$ for $p \in [0, 1]$, the variance has non-convex level sets, therefore is not elicitable.

\subsection{Problem Setting}\label{sec:problem_setting}
Let $\Gamma$ be an elicitable property. Following \cite{noarov2023statistical}, we consider the following online learning protocol: at each time $t = 1, \dots, T$, (a) the adversary selects a feature vector $x_{t} \in \cX$ and a distribution $Y_{t} \in \Delta(\cY)$ over the label space $\cY \subseteq \{0, 1\}$, and reveals $x_{t}$; (b) the forecaster randomly predicts $p_{t} \in [0, 1]$; (c) the adversary reveals $y_{t} \sim Y_{t}$. Throughout, (a) we assume that the adversary is adaptive, i.e., $x_t, Y_{t}$ are chosen with complete knowledge about the forecaster's algorithm and $p_{1}, \dots, p_{t - 1}$, and (b) the forecaster predicts $p_{t} \in \cZ \subset [0, 1]$ for all $t \in [T]$, where $\cZ = \{ z_1, \dots, z_{N}\}$ is a finite discretization of $[0, 1]$ and $z_{i} = \frac{i}{N}$ for all $i \in [N]$. We let $\mathbb{E}_{t}[.], \mathbb{P}_{t}[.]$ represent the conditional expectation, probability respectively, where the conditioning is over all randomness till time $t - 1$ (inclusive). Formally, letting $\mathscr{F}_{t - 1}$ be the filtration generated by the random variables $p_{1}, \dots, p_{t - 1}, y_{1}, \dots, y_{t - 1}$, we have $\mathbb{E}_{t}[.] \coloneqq \mathbb{E}[. | \mathscr{F}_{t - 1}], \mathbb{P}_{t}[.] \coloneqq \mathbb{P}[. | \mathscr{F}_{t - 1}]$ respectively.
 
 For a property $\Gamma$, an identification function $V: [0, 1] \times \cY \to [-1, 1]$
 takes as input the forecaster's prediction $p$, the adversary's chosen outcome $y$ and maps it to a real number in $[-1, 1]$ (assumed without any loss of generality). Similar to \cite{noarov2023statistical}, we assume that the marginal $V(p, Y_t) \coloneqq \mathbb{E}_{y \sim Y_t}[V(p, y) | \mathscr{F}_{t - 1}]$ is $\rho$-Lipschitz in $p$.
 \begin{assumption}\label{ass:assumption_identification}
     The function $V(p, Y_t)$ satisfies \begin{align}\label{eq:assumption_lipschitzness}
    \abs{V(p_2, Y_{t}) - V(p_1, Y_t)} \le \rho \abs{p_2 - p_1} \text{ for all } 0 \le p_1, p_2 \le 1.
\end{align}
Furthermore, $V(0, Y) \le 0$ and $V(1, Y) \ge 0$ for all $Y \in \Delta(\cY)$. 
\end{assumption}
For specific instantiations of \eqref{eq:assumption_lipschitzness}, observe that for means $V(p, Y_t) = p - \mathbb{E}_{t}[y]$ is trivially $1$-Lipschitz in $p$, whereas for quantiles $V(p, Y_t) = \mathbb{P}_{t}(y \le p) - q$ and \eqref{eq:assumption_lipschitzness} corresponds to the assumption that the CDF of $Y_t$ is $\rho$-Lipschitz, which is a standard assumption invoked in the relevant literature on quantile calibration \citep{jungbatch, garg2024oracle}. The second condition in \pref{ass:assumption_identification} is without any loss of generality; it is clearly true for means and quantiles and more generally can be interpreted as an underestimate/overestimate of the true property value. For instance, for expectiles, $V(p, y) = (p - y) (1 - \tau)$ if $p \ge y$ and $V(p, y) = (p - y) \tau$ if $p \le y$. Clearly, $V(p, Y)$ is Lipschitz in $p$ and satisfies $V(0, Y) \le 0$ and $V(1, Y) \ge 0$ for all $Y \in \Delta(\cY)$. Similarly, \pref{ass:assumption_identification} can be verified to be true for raw moments, $\Gamma(P) = \mathbb{P}(y = 1)$ under specific proper losses such as Tsallis entropy with $\alpha \ge 3$, spherical loss, etc.

For a bounded hypothesis class $\cF \subset [-1, 1] ^ {\cX}$, $\ell_{r}$-multicalibration ($\mcal_{\Gamma, r}(\cF)$) and $\ell_{r}$-swap multicalibration  ($\smcal_{\Gamma, r}(\cF)$) errors $(r \ge 1)$ are given by \eqref{eq:mcal_def_intro} and \eqref{eq:smcal_def_intro} respectively. 
The goal of the forecaster is to make predictions $p_1, \dots, p_{T}$ such that $\smcal_{\Gamma, r}(\cF)$ is minimized.

\subsection{Regret Minimization}
As mentioned, regret minimization algorithms are important subroutines of our final algorithm.
Consider the following interaction between an algorithm and adversary over $T$ rounds: at each time $t \in [T]$, the algorithm takes an action $a_{t} \in \cA$ and simultaneously the adversary reveals an outcome $y_{t} \in \cY$; the algorithm observes $y_{t}$ and incurs loss $\ell(a_{t}, y_t)$ for a known loss function $\ell: \cA \times \cY \to \Rn$. The \emph{external regret} $\mathsf{Reg} \coloneqq \sum_{t = 1} ^ {T} \ell(a_t, y_t) - \inf_{a \in \cA} \sum_{t = 1} ^ {T} \ell(a, y_t)$ compares the cumulative loss incurred by the algorithm with that incurred by the best fixed action chosen in hindsight. Compared to external regret, \emph{swap regret} is a stronger benchmark that compares the total loss incurred by the algorithm with that of the best swap (or modification) function $\nu: \cA \to \cA$ chosen in hindsight, i.e., $\sreg \coloneqq \sup_{\nu: \cA \to \cA} \sum_{t = 1} ^ {T} \ell(a_t, y_t) - \ell(\nu(a_t), y_t)$.

\paragraph{Expert problem.} The $K$-expert problem is a canonical problem in online learning, where at each time $t \in [T]$, an online algorithm outputs a probability distribution $w_{t} \in \Delta([K])$ 
over the experts and subsequently the adversary reveals a gain (or reward) vector $\ell_{t} \in [-1, 1] ^ {K}$, where $\ell_{t}(i)$ is the gain incurred by expert $i$. The regret $\mathsf{Reg} \coloneqq \max_{i^\star \in [K]} \sum_{t = 1} ^ {T} \ell_{t}(i^\star) -  \sum_{t = 1} ^ {T} \ip{w_{t}}{\ell_{t}}$ measures the difference between the gain of the best fixed expert chosen in hindsight and the expected gain of the algorithm.

\paragraph{Sequential Rademacher Complexity.} A $\cX$-valued binary tree is one in which all nodes have values that lie in $\cX$. For a $\cX$-valued tree $\x$ with depth $n$, a path from the root to a leaf can be represented as $\boldsymbol{\epsilon} = (\epsilon_1, \dots, \epsilon_{n})$, where $\epsilon_{i} \in \{-1, 1\}$ and $-1, +1$ correspond to traversing left, right respectively ($\epsilon_{n}$ is clearly irrelevant). Moreover, $\x$ can be represented by a sequence of mappings $\x_{1}, \dots, \x_{n}$, where $\x_{i}: \{-1, 1\} ^ {i - 1} \to \cX$ represents the value of a node at the $i$-th level of $\x$. For notational convenience, we represent $\x_{i}({\epsilon_{1}, \dots, \epsilon_{i - 1}})$ with $\x_{i}(\boldsymbol{\epsilon})$, with the understanding that $\epsilon_{i}, \dots, \epsilon_{n}$ are irrelevant when describing a node at the $i$-th level. 
\begin{definition}[Sequential Rademacher Complexity]
    For a hypothesis class $\cF$ and a $\cX$-valued tree $\x$ with depth $n$, the conditional sequential Rademacher complexity $\Re^{\text{seq}}(\cF, n; \x)$ is defined as \begin{align*}
        \Re^{\text{seq}}(\cF, n; \x) \coloneqq \frac{1}{n} \sup_{f \in \cF} \mathbb{E}_{\boldsymbol{\epsilon}}\bigs{\sum_{i = 1} ^ {n} \epsilon_{i} f(\x_{i}(\boldsymbol{\epsilon}))},
    \end{align*} 
    where the expectation is taken over the uniform distribution over $\{-1, 1\} \times \dots \times \{-1, 1\}$.
    The (unconditional) sequential Rademacher complexity $\Re^{\text{seq}}(\cF, n) $ is then defined as \begin{align*}
        \Re^{\text{seq}}(\cF, n) \coloneqq \sup_{\x} \Re^{\text{seq}}(\cF, n; \x),
    \end{align*}
    where the supremum is taken over all $\cX$-valued trees of depth $n$.
\end{definition}
Similar to the role of Rademacher complexity in describing the learnability of a hypothesis class $\cF$ in the distributional setting, the sequential Rademacher complexity characterizes learnability in the online setting.

\section{Algorithm}\label{sec:algorithm}
In this section, we propose an algorithm for multicalibrating an elicitable property $\Gamma$ with an identification function $V$ that satisfies \pref{ass:assumption_identification}. We first propose an inefficient algorithm and subsequently propose an efficient algorithm by reducing to online agnostic learning. 

\subsection{Achieving Multicalibration Inefficiently}\label{subsec:inefficient_multicalibration}

As mentioned in \pref{sec:introduction}, we build upon the algorithm proposed by \cite{hu2024predict} in a different context.  For each $i \in [N - 1]$, let $\cI_{i} = \left[\frac{i - 1}{N}, \frac{i}{N}\right)$ and $\cI_{N} = \left[\frac{N - 1}{N}, 1\right]$ so that $\cI_{1}, \dots, \cI_{N}$ represents a partition of $[0, 1]$.  We consider a $2N\abs{\cF}$-expert problem
where for each $f \in \cF, i \in [N]$, and $\sigma \in \{\pm 1\}$, we define the gain function $\ell_{f, i, \sigma}: [0, 1] \times \cX \times \cY \to [-1, 1]$
corresponding to the expert characterized by the tuple $(f, i, \sigma)$ as ${\ell_{f, i, \sigma}(p, x, y) \coloneqq \sigma \ind{p \in \cI_i} f(x) V(p, y)}$. At each time $t$, \pref{alg:hp_algorithm} obtains a probability distribution $\{w_{t, f, i, \sigma}\}$ over the $2\abs{\cF}N$ experts via an expert problem subroutine
$\alg$, where $w_{t, f, i, \sigma}$ is the probability allotted to the expert corresponding to $(f, i, \sigma)$. As shall be shown in \pref{prop:distribution}, in \pref{line:obtain_distribution} we obtain a distribution $P_t$ satisfying $h_{t}(P_t) \le \frac{\rho}{T}$, where the function $h_{t}: \Delta([0, 1]) \to [-1, 1]$ is defined as
$${h_{t}(P) \coloneqq \sup_{y \in \cY} \mathbb{E}_{p \sim P} \bigs{\sum_{(f, i, \sigma)} w_{t, f, i, \sigma} \ell_{f, i, \sigma} (p, x_{t}, y) | \mathscr{F}_{t - 1}}}.$$ 
Subsequently, \pref{alg:hp_algorithm} samples $\tilde{p}_{t} \sim P_t$ and predicts $p_{t} = \frac{i_{t}}{N}$, where $i_{t} \in [N]$ is such that $\tilde{p}_{t} \in \cI_{i_t}$. Finally, on observing $y_{t}$, we feed the gain ${\ell_{t, f, i, \sigma} \coloneqq \mathbb{E}_{p \sim P_t} \bigs{\ell_{f, i, \sigma} (p, y_t, x_{t}) | \mathscr{F}_{t - 1}}}$ for each $(f, i, \sigma)$ to $\alg$. A subtle difference between $\ell_{f, i, \sigma}$ and $\ell_{t, f, i, \sigma}$ is that the latter represents the actual gain fed to $\alg$, while the former is an auxiliary function used to define the latter. 

\begin{algorithm}[!htb]
                    \caption{Multicalibration for an Elicitable Property (Computationally Inefficient Version)} 
                    \label{alg:hp_algorithm}
                    \textbf{Initialize:} An expert problem subroutine $\alg$; 
                    \begin{algorithmic}[1]
                            \STATE\textbf{for} $t = 1, \dots, T,$
                            \STATE\hspace{3mm}Receive context $x_{t}$;
                            \STATE\hspace{3mm}Obtain weights $\{w_{t, f, i, \sigma}\}$ from $\alg$;
                            \STATE\label{line:obtain_distribution}\hspace{3mm}Define the function $\Phi_{t}: [0, 1] \to [-1, 1]$ as $${\Phi_{t}(p) \coloneqq \sum_{(f, i, \sigma)} w_{t, f, i, \sigma} \sigma \ind{p \in \cI_{i}} f(x_t)}.$$ If $\Phi_{t}(0) > 0$, choose $P_t$ that is only supported on $0$; else if $\Phi_{t}(1) \le 0$, choose $P_t$ that is only supported on $1$; else choose $i \in \{0, \dots, T - 1\}$ such that $\Phi_{t}\bigc{\frac{i}{T}} \Phi_{t}\bigc{\frac{i + 1}{T}} \le 0$ and $P_t$ such that \begin{align*}
    P_t\bigc{\frac{i}{T}} = \frac{\abs{\Phi_{t}\bigc{\frac{i + 1}{T}}}}{\abs{\Phi_{t}\bigc{\frac{i}{T}}} + \abs{\Phi_{t}\bigc{\frac{i + 1}{T}}}}, \quad P_t\bigc{\frac{i + 1}{T}} = \frac{\abs{\Phi_{t}\bigc{\frac{i}{T}}}}{\abs{\Phi_{t}\bigc{\frac{i}{T}}} + \abs{\Phi_{t}\bigc{\frac{i + 1}{T}}}}.
\end{align*}

                            \STATE\hspace{3mm}Sample $\tilde{p}_{t} \sim P_t$ and predict $p_{t} = z_{i_t} = \frac{i_{t}}{N}$, where $i_{t} \in [N]$ is such that $\tilde{p}_{t} \in \cI_{i_t}$;
                            \STATE\hspace{3mm}Observe $y_{t} \in \cY$;
                            \STATE\hspace{3mm}For each $(f, i, \sigma)$, feed ${\ell_{t, f, i, \sigma} = \mathbb{E}_{p \sim P_t}\bigs{\ell_{f, i, \sigma}(p, x_t,  y_{t}) | \mathscr{F}_{t - 1}}}$ to $\alg$;
                            
                        \end{algorithmic}
\end{algorithm}
\begin{proposition}\label{prop:distribution}
\pref{alg:hp_algorithm} satisfies $h_{t}(P_t) \le \frac{\rho}{T}$ for all $t \in [T]$. 
\end{proposition}
\begin{proof}
When $\Phi_{t}(0) > 0$, we have $$h_{t}(P_t) = \sup_{Y \in \Delta(\cY)} \mathbb{E}_{y \sim Y} \bigs{\Phi_{t}(0) V(0, y)} = \sup_{Y \in \Delta(\cY)} \Phi_{t}(0) V(0, Y) \le 0,$$ where the inequality follows from \pref{ass:assumption_identification}. Similarly, when $\Phi_{t}(1) \le 0$, we have $$h_{t}(P_t) = \sup_{Y \in \Delta(\cY)} \mathbb{E}_{y \sim Y} [\Phi_{t} (1) V(1, y)] = \sup_{Y \in \Delta(\cY)} \Phi_{t}(1) V(1, Y) \le 0.$$ For the last case, we have \begin{align*}
    h_{t}(P_t) &= \sup_{Y \in \Delta(\cY)} \mathbb{E}_{p \sim P_t, y \sim Y} \bigs{\Phi_{t}(p) V(p, y)} \\
    &= \sup_{Y \in \Delta(\cY)}\mathbb{E}_{y \sim Y} \bigs{\frac{\Phi_{t}\bigc{\frac{i}{T}} \abs{\Phi_{t}\bigc{\frac{i + 1}{T}}}V\bigc{\frac{i}{T}, y}}{\abs{\Phi_{t}\bigc{\frac{i}{T}}} + \abs{\Phi_{t}\bigc{\frac{i + 1}{T}}}}  + \frac{\Phi_{t}\bigc{\frac{i + 1}{T}} \abs{\Phi_{t}\bigc{\frac{i}{T}}} V\bigc{\frac{i + 1}{T}, y}}{\abs{\Phi_{t}\bigc{\frac{i}{T}}} + \abs{\Phi_{t}\bigc{\frac{i + 1}{T}}}}} \\
    &\le \sup_{Y \in \Delta(\cY)}  \bigs{\frac{\abs{\Phi_{t}\bigc{\frac{i}{T}}} \abs{\Phi_{t}\bigc{\frac{i + 1}{T}}}}{\abs{\Phi_{t}\bigc{\frac{i}{T}}} + \abs{\Phi_{t}\bigc{\frac{i + 1}{T}}}} \cdot \abs{V\bigc{\frac{i}{T}, Y} - V\bigc{\frac{i + 1}{T}, Y}}} \le \frac{\rho}{T},
\end{align*}
where the last inequality follows since due to the Lipschitzness of $V(., Y)$ \eqref{eq:assumption_lipschitzness} and that $\Phi_t$ takes value in $[-1,1]$.
This completes the proof.
\end{proof}

Instantiating $\alg$ with the McMwC algorithm \citep{chen2021impossible}, we obtain the following lemma.
\begin{lemma}\label{lem:expert_minimization_alg}
    The McMwC algorithm \citep{chen2021impossible} ensures the following regret bound \begin{align*}
        \sum_{t = 1} ^ {T} \ell_{t, f, i, \sigma} - \sum_{t = 1} ^ {T} \sum_{(f', i', \sigma')} w_{t, f', i', \sigma'} \ell_{t, f', i', \sigma'} = \cO\bigc{\log (\abs{\cF} N T) + \sqrt{\bigc{\sum_{t = 1} ^ {T} \mathbb{P}_{t}(\tilde{p}_t \in \cI_{i})} \log (\abs{\cF} NT)}}
    \end{align*}
    against each expert $(f, i, \sigma)$ simultaneously.
\end{lemma}
Note that a canonical expert minimization algorithm such as Multiplicative Weights Update would incur $\tilde{\cO}(\sqrt{T})$ regret against each expert, which is not sufficient to achieve the sub-$\sqrt{T}$ rates that we seek for in this paper. Instead, the McMwC algorithm ensures a regret bound that depends on $\sqrt{\sum_{t = 1} ^ {T} \mathbb{P}_{t}(\tilde{p}_t \in \cI_i)}$, which can be potentially much smaller than $\sqrt{T}$. 
Moreover, compared to \cite[Lemma 5.4]{hu2024predict} which incurs a factor proportional to $\sqrt{\sum_{t = 1} ^ {T} \ind{\tilde{p}_t \in \cI_i}}$, we incur a factor of $\sqrt{\sum_{t = 1} ^ {T} \mathbb{P}_{t}(\tilde{p}_t \in \cI_i)}$ instead, since McMwC actually ensures a regret bound $\tilde{\cO}\bigc{\sqrt{\sum_{t = 1} ^ {T} \ell_{t, f, i, \sigma} ^ 2}}$ (ignoring additive constants) against the expert $(f, i, \sigma)$. For our choice of $\ell_{t, f, i, \sigma}$, it follows that $\ell_{t, f, i, \sigma} ^ 2 \le \mathbb{P}_{t}(\tilde{p}_t \in \cI_i)$, thereby introducing a different factor. 

\begin{restatable}{lemma}{HPDeviationInefficient}\label{lem:hp_deviation_inefficient}
    \pref{alg:hp_algorithm} ensures that \begin{align*}
        \sup_{f \in \cF} \abs{\sum_{t = 1} ^ {T} \ind{p_t = z_i} f(x_t) V(p_t, y_t)} = \tilde{\cO}\bigc{\frac{\rho n_i}{N} + \sqrt{n_i \log \frac{\abs{\cF}N}{\delta}} + \log \frac{\abs{\cF}N}{\delta}}
    \end{align*}
    for all $i \in [N]$ with probability at least $1 - \delta$.
\end{restatable}
\begin{proof} Recall that $n_{i} = \sum_{t = 1} ^ {T} \ind{p_t = z_i} = \sum_{t = 1} ^ {T} \ind{\tilde{p}_t \in \cI_i}$ is the number of time instants the prediction made is $p_{t} = z_i = \frac{i}{N}$. Fixing a $i \in [N], f \in \cF$, it follows
from the triangle inequality that $$\abs{\sum_{t = 1} ^ {T} \ind{p_t = z_i} f(x_t) V(p_t, y_t)} \le \abs{\sum_{t = 1} ^ {T} X_t} + \abs{\sum_{t = 1} ^ {T} U_t} + \abs{\sum_{t = 1} ^ {T} Z_t} + \abs{\sum_{t = 1} ^ {T} W_t},$$ where the sequences $\{X_{t}\}, \{U_{t}\}, \{Z_{t}\}$, and $\{W_{t}\}$ are defined as 
    \begin{align*}
        X_{t} &\coloneqq \ind{\tilde{p}_t \in \cI_i} f(x_t) \cdot \bigc{V(z_i, y_t) - \mathbb{E}_{t} [V(z_i, y_{t})]}, \\
        U_{t} &\coloneqq \ind{\tilde{p}_t \in \cI_i} f(x_t) \cdot  (\mathbb{E}_t [V(z_i, y_t)] - \mathbb{E}_{t} [V(\tilde{p}_t, y_t)]), \\
        Z_{t} &\coloneqq \ind{\tilde{p}_t \in \cI_i} f(x_t) \cdot V(\tilde{p}_t, y_t), \\
        W_{t} &\coloneqq \ind{\tilde{p}_t \in \cI_i} f(x_t) \cdot (\mathbb{E}_{t}[V(\tilde{p}_t, y_t)] - V(\tilde{p}_t, y_t)).
    \end{align*}
Note that conditioned on $\mathscr{F}_{t - 1}$, the random variables $p_{t}$ (or $\tilde{p}_{t}$) and $y_{t}$ are independent. In the next steps, we bound the sum of each sequence individually.
\paragraph{Bounding $\abs{\sum_{t = 1} ^ {T} U_t}$.} Since $V(., Y_t)$ is $\rho$-Lipschitz and $\abs{p_{t} - \tilde{p}_{t}} \le \frac{1}{N}$, we immediately obtain $\abs{U_{t}} \le \frac{\rho}{N}$. 

\paragraph{Bounding $\abs{\sum_{t = 1} ^ {T} X_t}$.} Clearly, $\abs{X_{t}} \le 2$ and $\mathbb{E}_{t}[X_{t}] = 0$; thus, $\{X_{t}\}$ is a bounded martingale difference sequence. Furthermore, the cumulative variance $V_{X} \coloneqq \sum_{t = 1} ^ {T} \mathbb{E}_{t}\bigs{X_{t} ^ {2}}$ can be bounded as $V_{X} \le 4\sum_{t = 1} ^ {T} \mathbb{P}_{t}(\tilde{p}_t \in \cI_i)$. Therefore, for a fixed $\mu \in [0, \frac{1}{2}]$ and $\delta \in [0, 1]$, it follows from \pref{lem:Freedman} that \begin{align}\label{eq:bound_X_t}
    \abs{\sum_{t = 1} ^ {T} X_{t}} \le 4\mu \sum_{t = 1} ^ {T} \mathbb{P}_{t}(\tilde{p}_{t} \in \cI_i) + \frac{1}{\mu} \log \frac{2}{\delta}
\end{align}
with probability at least $1 - \delta$. 
For the subsequent steps, we assume that $m \in \mathbb{N}$ is such that \begin{align}\label{eq:n_existence}
    \frac{1}{2 ^ {m}} \le \sqrt{\frac{\log \frac{2 (m + 1)}{\delta}}{T}} < \frac{1}{2 ^ {m - 1}},
\end{align} and partition the interval $\cJ = \bigs{\frac{1}{2}\sqrt{\frac{\log \frac{2(m + 1)}{\delta}}{T}}, \frac{1}{2}}$ into dyadic intervals as $\cJ = \cJ_{m} \cup \dots \cup \cJ_{1} \cup \cJ_{0}$, where $\cJ_{k} = \left[\frac{1}{2 ^ {k + 1}}, \frac{1}{2 ^ {k}}\right)$ for all $k \in [m - 1], \cJ_{0} = \bigcurl{\frac{1}{2}}$, and $\cJ_{m} = \left[\frac{1}{2}\sqrt{\frac{\log \frac{2(m + 1)}{\delta}}{T}}, \frac{1}{2 ^ {m}}\right)$. For each interval $\cJ_{k}$, we associate a parameter $\mu_{k} = \frac{1}{2 ^ {k + 1}}$. Taking a union bound over all $k \in \{0, \dots, m\}$, we obtain \begin{align*}
     \abs{\sum_{t = 1} ^ {T} X_{t}} \le 4\mu_{k} \sum_{t = 1} ^ {T} \mathbb{P}_{t}(\tilde{p}_{t} \in \cI_i) + \frac{1}{\mu_{k}} \log \frac{2(m + 1)}{\delta}
 \end{align*}
 with probability at least $1 - \delta$ (simultaneously for all $k$). 
Each interval corresponds to a condition on $\sum_{t = 1} ^ {T} \mathbb{P}_{t}(\tilde{p}_{t} \in \cI_i)$ for which the optimal \begin{align*}
    \mu = \frac{1}{2}\min\bigc{1, \sqrt{\frac{\log \frac{2(m + 1)}{\delta}}{\sum_{t = 1} ^ {T} \mathbb{P}_t (\tilde{p}_t \in \cI_i)}}}    
\end{align*}
falls within that interval. In particular, for each $k \in [m - 1]$, the interval $\cJ_{k}$ shall correspond to the condition that \begin{align}\label{eq:constaint}
     \frac{1}{2 ^ {k}} \le \sqrt{\frac{\log \frac{2 (m + 1)}{\delta}}{\sum_{t = 1} ^ {T} \mathbb{P}_{t}(\tilde{p}_{t} \in \cI_i)}} < \frac{1}{2 ^ {k - 1}} 
     \equiv 4 ^ {k - 1} \log \frac{2 (m + 1)}{\delta} < \sum_{t = 1} ^ {T} \mathbb{P}(\tilde{p}_{t} \in \cI_i) \le 4 ^ {k} \log \frac{2 (m + 1)}{\delta}.
 \end{align}
 Additionally, $\cJ_{0}$ shall represent the condition that $0 \le \sum_{t = 1} ^ {T} \mathbb{P}_{t}(\tilde{p}_{t} \in \cI_i) \le \log \frac{2 (m + 1)}{\delta}$. Finally, $\cJ_{m}$ shall correspond to $4 ^ {m - 1} \log \frac{2(m + 1)}{\delta} < \sum_{t = 1} ^ {T} \mathbb{P}(\tilde{p}_{t} \in \cI_i) \le T$.
 Next, we prove an uniform upper bound on $\abs{\sum_{t = 1} ^ {T} X_{t}}$ by analyzing the bound for each interval. Towards this end, let $k \in [m - 1]$. Then,
 \begin{align*}
     \abs{\sum_{t = 1} ^ {T} X_{t}} \le \frac{1}{2 ^ {k - 1}} \sum_{t = 1} ^ {T} \mathbb{P}_{t}(\tilde{p}_{t} \in \cI_i) + 2 ^ {k + 1} \log \frac{2 (m + 1)}{\delta} \le 6 \sqrt{\bigc{\sum_{t = 1} ^ {T} \mathbb{P}_t(\tilde{p}_t \in \cI_i)} \log \frac{2 (m + 1)}{\delta}},
 \end{align*}
 where the second inequality follows from \eqref{eq:constaint}. 
 Similarly, for $k = 0$ we have \begin{align*}
     \abs{\sum_{t = 1} ^ {T} X_{t}} \le 2\sum_{t = 1} ^ {T} \mathbb{P}_{t}(\tilde{p}_t \in \cI_i) + 2 \log \frac{2 (m + 1)}{\delta} \le 4\log \frac{2 (m + 1)}{\delta}.
 \end{align*}
 Finally, for $k = m$ we have \begin{align*}
     \abs{\sum_{t = 1} ^ {T} X_{t}} &\le \frac{1}{2 ^ {m - 1}} \sum_{t = 1} ^ {T} \mathbb{P}_{t}(\tilde{p}_{t} \in \cI_i) + 2 ^ {m + 1} \log \frac{2 (m + 1)}{\delta} \\
     &\le 2\sqrt{\frac{\log \frac{2 (m + 1)}{\delta}}{T}} \sum_{t = 1} ^ {T} \mathbb{P}_{t}(\tilde{p}_{t} \in \cI_i) + 2 ^ {m + 1} \log \frac{2 (m + 1)}{\delta} \\
     &\le 6\sqrt{\bigc{\sum_{t = 1} ^ {T} \mathbb{P}_t(\tilde{p}_t \in \cI_i)} \log \frac{2 (m + 1)}{\delta}},
 \end{align*}
 where the first inequality follows from \eqref{eq:n_existence}, while the second inequality follows from the condition characterizing $\cJ_m$. Combining the three bounds above, we have shown that with probability at least $1 - \delta$,
 \begin{align*}
     \abs{\sum_{t = 1} ^ {T} X_{t}} &\le 6 \sqrt{\bigc{\sum_{t = 1} ^ {T} \mathbb{P}_{t}(\tilde{p}_t \in \cI_i)} \log \frac{2 (m + 1)}{\delta}} + 4 \log \frac{2 (m + 1)}{\delta} \\
     &= \tilde{\cO}\bigc{\sqrt{\bigc{\sum_{t = 1} ^ {T} \mathbb{P}_{t}(\tilde{p}_t \in \cI_i)} \log \frac{1}{\delta}} + \log \frac{1}{\delta}},
 \end{align*}
 where the equality follows since $m = \tilde{\Theta}\bigc{\log T}$. Taking a union bound over all $f \in \cF, i \in [N]$, we obtain
 \begin{align*}
     \sup_{f \in \cF} \abs{\sum_{t = 1} ^ {T} \ind{\tilde{p}_t \in \cI_i} f(x_t) \bigc{V(z_i, y_t) - \mathbb{E}_{t} [V(z_i, y_{t})]}} = \tilde{\cO}\bigc{\sqrt{\bigc{\sum_{t = 1} ^ {T} \mathbb{P}_{t}(\tilde{p}_t \in \cI_i)} \log \frac{N\abs{\cF}}{\delta}} + \log \frac{N\abs{\cF}}{\delta}}
 \end{align*}
for all $i \in [N]$ simultaneously with probability at least $1 - \delta$.

\paragraph{Bounding $\abs{\sum_{t = 1} ^ {T} Z_t}$.} It follows from the definition of the gain function $\ell_{f, i, \sigma}$ that \begin{align*}
    \abs{\sum_{t = 1} ^ {T} Z_{t}} &= \max_{\sigma \in \{\pm 1\}} \sum_{t = 1} ^ {T} \ell_{f, i, \sigma}(\tilde{p}_t, x_t, y_t) \\
    &\le \max_{\sigma \in \{\pm 1\}} \sum_{t = 1} ^ {T} \ell_{t, f, i, \sigma} + \max_{\sigma \in \{\pm 1\}} \sum_{t = 1} ^ {T} \ell_{f, i, \sigma} (\tilde{p}_t, x_t, y_t) - \mathbb{E}_{\tilde{p}_t}[\ell_{f, i, \sigma} (\tilde{p}_t, x_t, y_t) | \mathscr{F}_{t - 1}].
    \end{align*}
    Invoking the regret guarantee of $\alg$ (\pref{lem:expert_minimization_alg}), we obtain \begin{align*}
        \max_{\sigma \in \{\pm 1\}} \sum_{t = 1} ^ {T} \ell_{t, f, i, \sigma} \le \sum_{t = 1} ^ {T} \sum_{(f, i, \sigma)} w_{t, f, i, \sigma} \ell_{t, f, i, \sigma} + \tilde{\cO}\bigc{\log (\abs{\cF} N) + \sqrt{\bigc{\sum_{t = 1} ^ {T} \mathbb{P}_{t}(\tilde{p}_t \in \cI_i)} \log (\abs{\cF}N)}}.
    \end{align*}
    Moreover, $\sum_{t = 1} ^ {T} \sum_{(f, i, \sigma)} w_{t, f, i, \sigma} \ell_{t, f, i, \sigma} \le \sum_{t = 1} ^ {T} h_t (P_t) = \cO(1)$. For a fixed $\sigma \in \{\pm 1\}$, the sequence $\ell_{f, i, \sigma} (\tilde{p}_t, x_t, y_t) - \mathbb{E}_{\tilde{p}_t}[\ell_{f, i, \sigma} (\tilde{p}_t, x_t, y_t) | \mathscr{F}_{t - 1}]$ is a martingale difference sequence. Repeating a similar analysis as that done for the sequence $\{X_{t}\}$, we obtain \begin{align*}
        \abs{\sum_{t = 1} ^ {T} Z_{t}} = \tilde{\cO}\bigc{\sqrt{\bigc{\sum_{t = 1} ^ {T} \mathbb{P}_{t}(\tilde{p}_t \in \cI_i)} \log \frac{\abs{\cF}N}{\delta}} + \log \frac{\abs{\cF}N}{\delta}}.
    \end{align*}
Taking a union bound over all $f \in \cF, i \in [N]$, we obtain that \begin{align*}
    \sup_{f \in \cF} \abs{\sum_{t = 1}^ {T} \ind{\tilde{p}_t \in \cI_i} f(x_t) V(\tilde{p}_t, y_t)} = \tilde{\cO} \bigc{\sqrt{\bigc{\sum_{t = 1} ^ {T} \mathbb{P}_{t}(\tilde{p}_t \in \cI_i)} \log \frac{\abs{\cF}N}{\delta}} + \log \frac{\abs{\cF}N}{\delta}}.
\end{align*} 
holds for all $i \in [N]$ with probability at least $1 - \delta$.
\paragraph{Bounding $\abs{\sum_{t = 1} ^ {T} W_t}$.} {Observe that $\{W_t\}$ is not necessarily a martingale difference sequence. However, since the distribution $P_t$ randomizes between at most 2 points and $V(., Y)$ is Lipschitz, $\abs{\mathbb{E}_{t}[W_t]} \le \frac{\rho}{T}$. This is because, \begin{align*}
    {\mathbb{E}_{t}[W_t] = f(x_t) \mathbb{E}_{\tilde{p}_t}\bigs{\ind{\tilde{p}_t \in \cI_i} (\mathbb{E}_{\tilde{p}_t}[V(\tilde{p}_t, Y_t)] - V(\tilde{p}_t, Y_{t}))},}
\end{align*} 
where we have simplified the expression by not mentioning the filtration $\mathscr{F}_{t - 1}$, which is still implicit in $\mathbb{E}_{\tilde{p}_t}[.]$. Subsequently, if $P_t$ picks $0$ or $1$ deterministically, $\mathbb{E}_t[W_t] = 0$. In the remaining case, let $P_t$ randomize between 2 points $\frac{j}{T}$ and $\frac{j + 1}{T}$ at time $t$. Then, if the sampled $\tilde{p}_t$ is $\frac{j}{T}$, we have
\begin{align*}
    \abs{\mathbb{E}_{\tilde{p}_t}[V(\tilde{p}_t, Y_t)] - V(\tilde{p}_t, Y_{t})} &= \abs{P_t\bigc{\frac{j}{T}} V\bigc{\frac{j}{T}, Y_t} + P_t\bigc{\frac{j + 1}{T}} V\bigc{\frac{j + 1}{T}, Y_t} - V\bigc{\frac{j}{T}, Y_t}} \\
    &= \abs{P_t\bigc{\frac{j + 1}{T}} \bigc{V\bigc{\frac{j + 1}{T}, Y_t} - V\bigc{\frac{j}{T}, Y_t}}} \le \frac{\rho}{T},
\end{align*}
where the inequality follows due to the Lipschitzness of $V(., Y)$ \eqref{eq:assumption_lipschitzness}. The above bound also holds if the sampled $\tilde{p}_{t} = \frac{j + 1}{T}$. Therefore, $\abs{\mathbb{E}_{t}[W_{t}]} \le \frac{\rho}{T}$. Equipped with this, we can bound $\sum_{t = 1} ^ {T} W_{t}$ by considering the martingale difference sequence $W_{t} - \mathbb{E}_{t}[W_t]$ instead and repeating a similar analysis using Freedman's inequality. Particularly, $\abs{W_t - \mathbb{E}_t[W_t]} \le 4$ since $\abs{W_t} \le 2$ and the conditional variance can be bounded as $\sum_{t = 1} ^ {T} \mathbb{E}_{t}[(W_t - \mathbb{E}_t[{W_t}]) ^ 2] = \sum_{t = 1} ^ {T} \mathbb{E}_t[W_t ^ 2] - (\mathbb{E}_t[W_t]) ^ 2 \le \sum_{t = 1} ^ T \mathbb{E}_t[W_t ^ 2] \le 4 \cdot \mathbb{P}_t(\tilde{p}_t \in \cI_i)$. Repeating a similar analysis as that done for the sequence $\{X_t\}$, we obtain \begin{align*}
    \abs{\sum_{t = 1} ^ {T} W_t - \mathbb{E}_t [W_t]} =  \tilde{\cO}\bigc{\sqrt{\bigc{\sum_{t = 1} ^ T \mathbb{P}_{t}(\tilde{p}_t \in \cI_i)} \log \frac{1}{\delta}} + \log \frac{1}{\delta}}.
\end{align*}
Since $\abs{\sum_{t = 1} ^ {T} W_t} \le \abs{\sum_{t = 1} ^ T W_t - \mathbb{E}_t[W_t]} + \abs{\sum_{t = 1} ^ {T} \mathbb{E}_t[W_t]}$ and $\abs{\mathbb{E}_t[W_t]} \le \frac{\rho}{T}$, we obtain \begin{align*}
    \abs{\sum_{t = 1} ^ {T} W_t} =  \tilde{\cO}\bigc{\sqrt{\bigc{\sum_{t = 1} ^ T \mathbb{P}_{t}(\tilde{p}_t \in \cI_i)} \log \frac{1}{\delta}} + \log \frac{1}{\delta}}.
\end{align*}
Taking a union bound over all $f \in \cF, i \in [N]$, we obtain that \begin{align*}
    \sup_{f \in \cF} \abs{\sum_{t = 1} ^ {T} \ind{\tilde{p}_t \in \cI_i} f(x_t)  (\mathbb{E}_{t}[V(\tilde{p}_t, y_t)] - V(\tilde{p}_t, y_t))} = \tilde{\cO}\bigc{\sqrt{\bigc{\sum_{t = 1} ^ T \mathbb{P}_{t}(\tilde{p}_t \in \cI_i)} \log \frac{\abs{\cF} N}{\delta}} + \log \frac{\abs{\cF} N}{\delta}}
\end{align*}
holds for all $i \in [N]$ with probability at least $1 - \delta$.
}

\paragraph{Relating $\sum_{t = 1} ^ {T} \mathbb{P}_t (\tilde{p}_t \in \cI_i)$ with $n_i$.} Next, we obtain a high probability bound that relates $\sum_{t = 1} ^ {T} \mathbb{P}_t (\tilde{p}_t \in \cI_i)$ with $n_i$. Consider the sequence $A_{t} \coloneqq \ind{\tilde{p}_t \in \cI_i} - \mathbb{P}_{t}(\tilde{p}_t \in \cI_i)$. Again, $\{A_{t}\}$ is a martingale difference sequence, therefore, repeating the exact same analysis as that for the sequence $\{X_{t}\}$, we obtain that \begin{align}\label{eq:deviation_ind_prob}
    \abs{\sum_{t = 1} ^ {T} \ind{\tilde{p}_t \in \cI_i} - \mathbb{P}_t (\tilde{p}_t \in \cI_i)} &= \tilde{\cO}\bigc{\sqrt{\bigc{\sum_{t = 1} ^ T \mathbb{P}_{t}(\tilde{p}_t \in \cI_i)} \log \frac{N}{\delta}} + \log \frac{N}{\delta}}.
\end{align}
holds for all $i \in [N]$ with probability at least $1 - \delta$. Let $\alpha \coloneqq \sum_{t = 1} ^ {T} \mathbb{P}_t (\tilde{p}_t \in \cI_i), \beta = \log \frac{N}{\delta}$, and $\gamma > 0$ be the constant (including logarithmic terms) hidden in the $\tilde{\cO}(.)$ notation. Then, \eqref{eq:deviation_ind_prob} can be expressed as $\abs{n_i - \alpha} \le \gamma (\sqrt{\alpha \beta} + \beta)$, therefore $\alpha$ satisfies $\alpha - \gamma \sqrt{\alpha\beta} - (n_{i} + \gamma \beta) \le 0$, which implies that $\sqrt{\alpha} \le \frac{\gamma \sqrt{\beta} + \sqrt{\gamma ^ 2 \beta + 4 n_i + 4 \gamma \beta}}{2} \le \gamma \sqrt{\beta} + \sqrt{n_i} + \sqrt{\gamma \beta} = \tilde{\cO}\bigc{\sqrt{\log \frac{N}{\delta}} + \sqrt{n_i}}$. 

\paragraph{Combining everything.} For all $i \in [N]$, with probability at least $1 - \delta$, we have
\begin{align*}
    \sup_{f \in \cF} \abs{\sum_{t = 1} ^ {T} \ind{\tilde{p}_t \in \cI_i} f(x_t) \bigc{V(z_i, y_t) - \mathbb{E}_{t} [V(z_i, y_{t})]}} &= \tilde{\cO}\bigc{\sqrt{n_i \log \frac{\abs{\cF} N}{\delta}} + \log \frac{\abs{\cF} N}{\delta}}, \\
     \max_{\sigma \in \{\pm 1\}} \sup_{f \in \cF} \abs{\sum_{t = 1} ^ {T} \ell_{f, i, \sigma} (\tilde{p}_t, x_t, y_t) - \mathbb{E}_{\tilde{p}_t}[\ell_{f, i, \sigma} (\tilde{p}_t, x_t, y_t) | \mathscr{F}_{t - 1}]} &= \tilde{\cO}\bigc{\sqrt{n_i \log \frac{\abs{\cF} N}{\delta}} + \log \frac{\abs{\cF} N}{\delta}}, \\
    \sup_{f \in \cF} \abs{\sum_{t = 1} ^ T \ind{\tilde{p}_t \in \cI_i} f(x_t) (\mathbb{E}_{t}[V(\tilde{p}_t, y_t)] - V(\tilde{p}_t, y_t))} &= \tilde{\cO}\bigc{\sqrt{n_i \log \frac{\abs{\cF} N}{\delta}} + \log \frac{\abs{\cF} N}{\delta}}.
\end{align*}
Combining the three bounds above with the bound on $\abs{\sum_{t = 1} ^ {T} U_t}$, we obtain
\begin{align*}
    \sup_{f \in \cF} \abs{\sum_{t = 1} ^ {T} \ind{p_t = z_i} f(x_t) V(p_t, y_t)} = \tilde{\cO}\bigc{\frac{\rho n_i}{N} + \sqrt{n_i \log \frac{\abs{\cF} N}{\delta}} + \log \frac{\abs{\cF} N}{\delta}}.
\end{align*}
This completes the proof.\end{proof}
Equipped with \pref{lem:hp_deviation_inefficient}, we prove the main result of this subsection.
\begin{theorem}\label{thm:smcal_bound_inefficient}
    For a fixed $r \ge 2$, \pref{alg:hp_algorithm} achieves $$\smcal_{\Gamma, r}(\cF) = 
        \tilde{\cO}\bigc{{\rho ^ {r} T ^ {\frac{1}{r + 1}} + T^{\frac{1}{r + 1}} \bigc{\log \frac{\abs{\cF}}{\delta}} ^ {\frac{r}{2}}}}$$ 
        with probability at least $1 - \delta$. Consequently, for $r \in [1, 2)$, \pref{alg:hp_algorithm} achieves \begin{align*}
        \smcal_{\Gamma, r}(\cF) = \tilde{\cO}\bigc{{\rho ^ {r} T ^ {1 - \frac{r}{3}} + T ^ {1 - \frac{r}{3}} \bigc{\log \frac{\abs{\cF}}{\delta}} ^ {\frac{r}{2}}}}
    \end{align*}
    with probability at least $1 - \delta$.
\end{theorem}
\begin{proof}
    Fix a $r \ge 2$. We bound $\smcal_{\Gamma, r}(\cF)$ using the high probability bound in \pref{lem:hp_deviation_inefficient}. From \pref{lem:hp_deviation_inefficient}: \begin{align}\label{eq:hp_deviation_inefficient}
        \frac{1}{n_i}\sup_{f \in \cF} \abs{\sum_{t = 1} ^ {T} \ind{p_t = z_i} f(x_t) V(p_t, y_t)} \le \begin{cases}
            \tilde{\cO}\bigc{\sqrt{\frac{\log \frac{\abs{\cF} N}{\delta}}{n_i}} + \frac{\rho}{N}} & i \in \cS, \\
            1 & i \in \bar{\cS},
        \end{cases}
    \end{align}
    where $\cS \coloneqq \bigcurl{i \in [N]; n_{i} = \Omega\bigc{\log \frac{\abs{\cF} N}{\delta}}}$. Therefore, \begin{align*}
        \smcal_{\Gamma, r}(\cF) &= \sum_{i \in [N]} n_i \bigc{\frac{1}{n_i}\sup_{f \in \cF} \abs{\sum_{t = 1} ^ {T} \ind{p_t = z_i} f(x_t) V(p_t, y_t)}} ^ {r} \\
        &= \sum_{i \in \cS} n_{i} \cdot \tilde{\cO}\bigc{{\frac{\rho ^ {r}}{N ^ {r}} + \bigc{\frac{\log \frac{\abs{\cF}N}{\delta}}{n_i}} ^ {\frac{r}{2}}}} + \cO\bigc{\abs{\bar{\cS}} \log \frac{\abs{\cF}N}{\delta}} \\
        &= \tilde{\cO}\bigc{{\rho ^ {r}\frac{T}{N ^ {r}} + N \bigc{\log \frac{\abs{\cF}N}{\delta}} ^ {\frac{r}{2}}}} 
\end{align*}
where the first equality follows since $(u + v) ^ {r} \le 2 ^ {r - 1} (u ^ {r} + v ^ {r})$, which holds for all $u, v$ via Jensen's inequality applied to the function $x \to x^{r}$, which is convex for $r \ge 1$; we incur the $\cO\bigc{\abs{\bar{\cS}} \log \frac{\abs{\cF}N}{\delta}}$ term due to \eqref{eq:hp_deviation_inefficient}, and since $n_i = \cO\bigc{\log \frac{\abs{\cF}N}{\delta}}$ for $i \in \bar{\cS}$. To obtain the second equality, we bound $n_{i} ^ {1 - \frac{r}{2}} < 1$. Setting $N = \Theta(T^{\frac{1}{r + 1}})$ to balance the terms $\frac{T}{N ^ {r}}$ and $N$, we obtain the desired bound on $\smcal_{\Gamma, r}(\cF)$. For $r \in [1, 2)$, since $\smcal_{\Gamma, r}(\cF) \le T ^ {1 - \frac{r}{2}} (\smcal_{\Gamma, 2}(\cF)) ^ {\frac{r}{2}}$ by H\"older's inequality, we obtain \begin{align*}
    \smcal_{\Gamma, r}(\cF) = \tilde{\cO}\bigc{T ^ {1 - \frac{r}{2}} \cdot \bigc{\rho ^ 2 T ^ {\frac{1}{3}} + T ^ \frac{1}{3} \log \frac{\abs{\cF}}{\delta}} ^ {\frac{r}{2}}} = \tilde{\cO}\bigc{{\rho ^ {r} T ^ {1 - \frac{r}{3}} + T ^ {1 - \frac{r}{3}} \bigc{\log \frac{\abs{\cF}}{\delta}} ^ {\frac{r}{2}}}},
\end{align*}
where the second equality follows since $(u + v) ^ {\frac{r}{2}} \le \sqrt{2 ^ {r - 1} (u ^ {r} + v ^ {r})} \le 2 ^ {\frac{r - 1}{2}} (u ^ {\frac{r}{2}} + v ^ {\frac{r}{2}})$. This completes the proof.
\end{proof}
\subsection{Achieving Oracle-Efficient Multicalibration}\label{subsec:oracle_efficient_multicalibration}

Clearly, \pref{alg:hp_algorithm} is inefficient since it requires enumerating over $\cF$. In this section, we propose an oracle-efficient algorithm by reducing to online agnostic learning.
\DefOnlineAgnostic* 
For a randomized online agnostic learner that achieves \eqref{eq:WAL_regret} in expectation, due to linearity of \eqref{eq:WAL_regret} in the prediction $q_t(x_t)$, it is sufficient to predict the expectation of $q_{t}(x_t)$ at each time $t$ to achieve the same guarantee, therefore, without any loss of generality, {we may assume that $\owal$ is deterministic.}

Towards making \pref{alg:hp_algorithm} efficient, we instantiate $2N$ copies of $\owal$$\{\owal_{i, \sigma}\}_{i \in [N], \sigma \in \{\pm 1\}}$ and only maintain a probability distribution over the experts parametrized by $(i, \sigma)$. At each time $t$, each $\owal_{i, \sigma}$ provides an auditing function $q_{t, i, \sigma}$, which we utilize along with the probability distribution $\{w_{t, i, \sigma}\}$ over the experts to define the distribution
$${h_{t}(P) \coloneqq \sup_{y \in \cY} \mathbb{E}_{p \sim P} \bigs{\sum_{(i, \sigma)} w_{t, i, \sigma} \sigma \ind{p \in \cI_{i}} q_{t, i, \sigma}(x_{t}) V(p, y) | \mathscr{F}_{t - 1}}}.$$
Next, we obtain $P_t$ satisfying $h_{t}(P_t) \le \frac{\rho}{T}$ similar to \pref{line:obtain_distribution} in \pref{alg:hp_algorithm}. Subsequently, after predicting $p_{t}$ and observing $y_t$, we feed the gain ${\ell_{t, i, \sigma} \coloneqq \mathbb{E}_{p \sim P_t}\bigs{\ell_{t, i, \sigma}(p, y_t) | \mathscr{F}_{t - 1}}}$ to $\alg$, where the gain function $\ell_{t, i, \sigma}(p, y)$ is defined as ${\ell_{t, i, \sigma}(p, y) \coloneqq \sigma \ind{p \in \cI_i} q_{t, i, \sigma}(x_t) V(p, y)}$, and the outcome ${\sigma V(p_t, y_t)}$ to $\owal_{i_{t}, \sigma}$ for each $\sigma \in \{-1,+1\}$ (no feedback to $\owal_{i, \sigma}$ with $i\neq i_t$). The full algorithm is summarized in \pref{alg:hp_algorithm_efficient}.
\begin{algorithm}[!htb]
                    \caption{Multicalibration for an Elicitable Property (Oracle-Efficient Version)} 
                    \label{alg:hp_algorithm_efficient}
                    \textbf{Initialize:} An expert problem subroutine $\alg$, online agnostic learner ($\owal_{i, \sigma}$) for each $(i, \sigma)$; 
                    \begin{algorithmic}[1]
                            \STATE\textbf{for} $t = 1, \dots, T,$
                            \STATE\hspace{3mm}Receive context $x_t \in \cX$; \STATE\hspace{3mm}Obtain weights $\{w_{t, i, \sigma}\}$ from $\alg$;
                            \STATE\hspace{3mm}For each $(i, \sigma)$, if no outcome was feed to $\owal_{i, \sigma}$ in the last round, let $q_{t, i, \sigma}$ be the {most recent} output of $\owal_{i, \sigma}$; else obtain a new test function from $\owal_{i, \sigma}$ as $q_{t, i, \sigma}$; 
                            \STATE\hspace{3mm}Define the function $\Phi_{t}: [0, 1] \to [-1, 1]$ as $$\Phi_{t}(p) \coloneqq \sum_{(i, \sigma)} w_{t, i, \sigma} \sigma \ind{p \in \cI_{i}} q_{t, i, \sigma}(x_t).$$ If $\Phi_{t}(0) > 0$, choose a distribution $P_t$ that is only supported on $0$; else if $\Phi_t(1) \le 0$, choose $P_t$ that is only supported on $1$; else choose $i \in \{0, \dots, T - 1\}$ such that $\Phi_{t}\bigc{\frac{i}{T}} \Phi_{t}\bigc{\frac{i + 1}{T}} \le 0$ and $P_t$ such that \begin{align*}
    P_t\bigc{\frac{i}{T}} = \frac{\abs{\Phi_{t}\bigc{\frac{i + 1}{T}}}}{\abs{\Phi_{t}\bigc{\frac{i}{T}}} + \abs{\Phi_{t}\bigc{\frac{i + 1}{T}}}}, \quad P_t\bigc{\frac{i + 1}{T}} = \frac{\abs{\Phi_{t}\bigc{\frac{i}{T}}}}{\abs{\Phi_{t}\bigc{\frac{i}{T}}} + \abs{\Phi_{t}\bigc{\frac{i + 1}{T}}}}.
\end{align*}
                            \STATE\hspace{3mm}Sample $\tilde{p}_{t} \sim P_t$ and predict $p_{t} = \frac{i_{t}}{N}$, where $i_{t} \in [N]$ is such that $\tilde{p}_{t} \in \cI_{i_t}$;
                            \STATE\hspace{3mm}Observe $y_{t}$;
                            \STATE\hspace{3mm}For each $(i, \sigma)$, feed the gain ${\ell_{t, i, \sigma} = \mathbb{E}_{p \sim P_t}\bigs{\ell_{t, i, \sigma}(p, y_t) | \mathscr{F}_{t - 1}}}$ to $\alg$, where ${\ell_{t, i, \sigma}(p, y) = \sigma \ind{p \in \cI_i} q_{t, i, \sigma}(x_t) V(p, y)}$;
                            \STATE\hspace{3mm}Feed the outcome $\sigma V(p_{t}, y_{t})$ to $\owal_{i_t, \sigma}$ for each $\sigma \in \{-1,+1\}$;
                        \end{algorithmic}
\end{algorithm}
 
The following lemma is in a similar spirit as \pref{lem:hp_deviation_inefficient} in \pref{subsec:inefficient_multicalibration}.
\begin{restatable}
{lemma}{HPDeviationEfficient}\label{lem:hp_bound_deviation}
    \pref{alg:hp_algorithm_efficient} ensures that \begin{align*}
        \sup_{f \in \cF} \abs{\sum_{t = 1} ^ {T} \ind{p_t = z_i} f(x_t) V(p_t, y_t)} = \tilde{\cO}\bigc{\frac{\rho n_i}{N} + \sqrt{n_i \log \frac{N}{\delta}} + \log \frac{N}{\delta} + \mathsf{Reg}(\cF, n_i)}
    \end{align*}
    for all $i \in [N]$ with probability at least $1 - \delta$.
\end{restatable}
\begin{proof}
    For convenience, let $\kappa_{t, i, \sigma} \coloneqq \sigma \ind{p_{t} = z_i} V(p_t, y_t)$. Fix a $i \in [N]$ and failure probability $\delta$. We begin by realizing that \begin{align*}\sup_{f \in \cF} \abs{\sum_{t = 1} ^ {T} \ind{p_t = z_i} f(x_t) V(p_t, y_t)} &= \sup_{\sigma \in \{\pm 1\}, f \in \cF} \sum_{t = 1} ^ {T} \sigma \ind{p_t = z_i} f(x_t) V(p_t, y_t) \\
    &= \sup_{\sigma \in \{\pm 1\}, f \in \cF} \sum_{t = 1} ^ {T} \kappa_{t, i, \sigma} f(x_{t}) \\
    &\le \sup_{\sigma \in \{\pm 1\}} \sum_{t = 1} ^ {T} \kappa_{t, i, \sigma} q_{t, i, \sigma}(x_{t}) +  \mathsf{Reg} (\cF, n_i),
    \end{align*}
    where the inequality follows from \eqref{eq:WAL_regret}, since $\owal_{i, \sigma}$ is updated exactly on the rounds where the prediction made is $p_{t} = z_i$.
    Next, we obtain a high probability bound on the quantity $\sum_{t = 1} ^ {T} \kappa_{t, i, \sigma} q_{t, i, \sigma}(x_{t})$ (for a fixed $\sigma$) by proceeding in a similar manner as \pref{lem:hp_deviation_inefficient}. We begin with the decomposition $\sum_{t = 1} ^ {T} \kappa_{t, i, \sigma} q_{t, i, \sigma}(x_{t}) = \sum_{t = 1} ^ {T} X_{t} + U_t + Z_t + W_t$, where the sequences $\{X_{t}\}, \{U_{t}\}, \{V_{t}\}, \{W_{t}\}$ are defined as \begin{align*}X_{t} &\coloneqq \sigma \ind{\tilde{p}_t \in \cI_i} q_{t, i, \sigma}(x_t) \cdot \bigc{V(z_i, y_t) - \mathbb{E}_{t} [V(z_i, y_{t})]}, \\
        U_{t} &\coloneqq \sigma \ind{\tilde{p}_t \in \cI_i} q_{t, i, \sigma}(x_{t}) \cdot  (\mathbb{E}_t [V(z_i, y_t)] - \mathbb{E}_{t} [V(\tilde{p}_t, y_t)]), \\
        Z_{t} &\coloneqq \sigma \ind{\tilde{p}_t \in \cI_i} q_{t, i, \sigma}(x_t) \cdot V(\tilde{p}_t, y_t), \\
        W_{t} &\coloneqq \sigma \ind{\tilde{p}_t \in \cI_i} q_{t, i, \sigma}(x_t) \cdot (\mathbb{E}_{t}[V(\tilde{p}_t, y_t)] - V(\tilde{p}_t, y_t)).\end{align*}
 Clearly, $\abs{U_t} \le \frac{\rho}{N}$. Since $q_{t, i, \sigma}$ is $\mathscr{F}_{t - 1}$ measurable, $\{X_{t}\}$ is a martingale difference sequence, where as $\{W_{t}\}$ is not necessarily so. Following the same reasoning as in the proof of \pref{lem:hp_deviation_inefficient}, we obtain \begin{align*}
    \sum_{t = 1} ^ {T} X_{t} &= \tilde{\cO}\bigc{\sqrt{\bigc{\sum_{t = 1} ^ T \mathbb{P}_{t}(\tilde{p}_t \in \cI_i)} \log \frac{1}{\delta}} + \log \frac{1}{\delta}}, \\
    \sum_{t = 1} ^ {T} W_{t} &= \tilde{\cO}\bigc{\sqrt{\bigc{\sum_{t = 1} ^ T \mathbb{P}_{t}(\tilde{p}_t \in \cI_i)} \log \frac{1}{\delta}} + \log \frac{1}{\delta}}
\end{align*}
with probability at least $1 - \delta$. To bound $\sum_{t = 1} ^ {T} Z_t$, we first decompose $
\sum_{t = 1} ^ {T} Z_{t} = \sum_{t = 1} ^ {T} \ell_{t, i, \sigma}(\tilde{p}_t, y_t) = \sum_{t = 1} ^ {T} \ell_{t, i, \sigma} + \sum_{t = 1} ^ {T} \ell_{t, i, \sigma} (\tilde{p}_t, y_t) - \mathbb{E}_{\tilde{p}_t}[\ell_{t, i, \sigma} (\tilde{p}_t, y_t) | \mathscr{F}_{t - 1}]$.
It follows from the regret guarantee of $\alg$ (\pref{lem:expert_minimization_alg}) that \begin{align*}
   \sum_{t = 1} ^ {T} \ell_{t, i, \sigma} &\le \sum_{t = 1} ^ {T} \sum_{(i, \sigma)} w_{t, i, \sigma} \ell_{t, i, \sigma} + \cO\bigc{\log (NT) + \sqrt{\bigc{\sum_{t = 1} ^ {T} \mathbb{P}_{t}(\tilde{p}_t \in \cI_i)} \log (NT)}} \\
    &= \tilde{\cO}\bigc{\log N + \sqrt{\bigc{\sum_{t = 1} ^ {T} \mathbb{P}_{t}(\tilde{p}_t \in \cI_i)} \log N}},
\end{align*}
since for each $t$, $\sum_{(i, \sigma)} w_{t, i, \sigma} \ell_{t, i, \sigma} \le h_t(P_t) \le \frac{\rho}{T}$ by the choice of $P_t$ (exact same reasoning as \pref{prop:distribution}). The deviation $\sum_{t = 1} ^ {T} \ell_{t, i, \sigma} (\tilde{p}_t, y_t) - \mathbb{E}_{\tilde{p}_t}[\ell_{t, i, \sigma} (\tilde{p}_t, y_t) | \mathscr{F}_{t - 1}]$ can be bounded similar to \pref{lem:hp_deviation_inefficient}; we obtain
\begin{align*}
    \sum_{t = 1} ^ {T} \ell_{t, i, \sigma} (\tilde{p}_t, y_t) - \mathbb{E}_{\tilde{p}_t}[\ell_{t, i, \sigma} (\tilde{p}_t, y_t) | \mathscr{F}_{t - 1}] = \tilde{\cO}\bigc{\sqrt{\bigc{\sum_{t = 1} ^ T \mathbb{P}_{t}(\tilde{p}_t \in \cI_i)} \log \frac{1}{\delta}} + \log \frac{1}{\delta}}.
\end{align*}
Combining everything and taking a union bound over $i \in [N]$, we obtain 
\begin{align*}
    \sum_{t = 1} ^ {T} \kappa_{t, i, \sigma} q_{t, i, \sigma}(x_t) = \tilde{\cO}\bigc{\frac{\rho n_i}{N} + \sqrt{n_i \log \frac{N}{\delta}} + \log \frac{{N}}{\delta}}
\end{align*}
with probability at least $1 - \delta$. This completes the proof.
\end{proof}

\begin{theorem}\label{thm:smcal_general_result}
     {Fix a $r \ge 1$ and assume that there exists an $\owal$ for which $\mathsf{Reg}(\cF, n) = \tilde{\cO}(n^{\alpha} \comp(\cF))$, where $\alpha \in [0, 1)$ and $\comp(\cF)$ is a complexity measure of $\cF$ that is independent of $n$. Then, for $r \ge 2$, \pref{alg:hp_algorithm_efficient} achieves \begin{align*}
        \smcal_{\Gamma, r}(\cF) = \tilde{\cO}\bigc{{\rho ^ {r} T ^ {\frac{1}{r + 1}} + T^{\frac{1}{r + 1}} \bigc{\log \frac{1}{\delta}} ^ {\frac{r}{2}} + T ^ {1 - r + \frac{r}{r + 1} + \frac{\alpha r ^ {2}}{r + 1}} \comp(\cF) ^ {r} + T ^ {\frac{1}{r + 1}} \comp(\cF) ^ {r}}}
    \end{align*}
    with probability at least $1 - \delta$. Consequently, for $r \in [1, 2)$, \pref{alg:hp_algorithm_efficient} achieves 
    \begin{align*}
        \smcal_{\Gamma, r}(\cF) = \tilde{\cO}\bigc{{\rho ^ {r} T ^ {1 - \frac{r}{3}} + T ^ {1 - \frac{r}{3}} \bigc{\log \frac{1}{\delta}} ^ {\frac{r}{2}} + T ^ {1 + \frac{2r (\alpha - 1)}{3}} \comp(\cF) ^ {r} + T ^ {1 - \frac{r}{3}} \comp(\cF) ^ {r}}}
    \end{align*}
    with probability at least $1 - \delta$.}
\end{theorem}
\begin{proof}
    Fix a $r \ge 2$. It follows from \pref{lem:hp_bound_deviation} that
    \begin{align}\label{eq:temp}
    \frac{1}{n_i} \sup_{f \in \cF} \abs{\sum_{t = 1} ^ {T} \ind{p_t = z_i} f(x_{t}) V(p_t, y_t)} &= \begin{cases}
        \tilde{\cO}\bigc{\frac{\rho}{N} + \sqrt{\frac{\log \frac{N}{\delta}}{n_i}} + \frac{\reg(\cF, n_i)}{n_i}} & i \in \cS, \\
        1 & \text{otherwise},
    \end{cases}
\end{align}
where $\cS \coloneqq \{i \in [N]; n_i = \Omega(\log \frac{N}{\delta})\}$. Therefore, 
\begin{align*}
    \smcal_{\Gamma, r}(\cF) &= \sum_{i \in [N]} n_{i} \bigc{\frac{1}{n_i} \sup_{f \in \cF} \abs{\sum_{t = 1} ^ {T} \ind{p_t = z_i} f(x_{t}) V(p_t, y_t)}}^{r} \\
    &= \sum_{i \in \cS} n_{i} \cdot \tilde{\cO}\bigc{{\frac{\rho ^ {r}}{N ^ {r}} + \bigc{\frac{\log \frac{N}{\delta}}{n_{i}}} ^ {\frac{r}{2}} + \bigc{\frac{\reg(\cF, n_i)}{n_{i}}} ^ {r}}} + \cO\bigc{\abs{\bar{\cS}} \log \frac{N}{\delta}} \\
    &= \tilde{\cO}\bigc{{\rho ^ {r} \frac{T}{N ^ r} + N \bigc{\log \frac{N}{\delta}} ^ {\frac{r}{2}} + \sum_{i \in \cS} n_{i} \bigc{\frac{\reg(\cF, n_i)}{n_i}} ^ {r}}},
\end{align*}
where the first equality follows since $(u + v + w) ^ {r} \le 3 ^ {r - 1} (u ^ {r} + v ^ {r} + w ^ {r})$ via Jensen's inequality applied to the function $x \to x ^ {r}$, which is convex for $r \ge 1$; the $\cO\bigc{\abs{\bar{\cS}} \log \frac{N}{\delta}}$ term is because of \eqref{eq:temp}, and since $n_{i} = \cO(\log \frac{N}{\delta})$ for $i \in \bar{\cS}$. To obtain the second equality, we bound $n_{i} ^ {1 - \frac{r}{2}} < 1$. Next, observe that as a function of $n_{i}$, $\psi(n_i) \coloneqq n_i \bigc{\frac{\reg(\cF, n_i)}{n_i}} ^ {r}$ is either concave or convex in $n_{i}$. Particularly, since the dependence of $n_{i}$ in $\reg(\cF, n_i)$ is $n_{i} ^ {\alpha}$ for some $\alpha \in [0, 1)$, the dependence of $n_{i}$ in $\psi(n_i)$ is $n_{i} ^ {\beta}$, where $\beta = 1 + r (\alpha - 1) < 1$. If $\psi(n_i)$ is concave in $n_i$ (corresponds to $\beta \in [0, 1)$), we have \begin{align*}
    \sum_{i \in \cS} \psi(n_i) \le \sum_{i \in [N]} \psi(n_i) \le N \psi \bigc{\frac{1}{N} \sum_{i = 1} ^ {N} n_i} = N \psi\bigc{\frac{T}{N}} = T \bigc{\frac{N}{T} \reg\bigc{\cF, \frac{T}{N}}} ^ {r},
\end{align*}
where the second inequality follows from Jensen's inequality. On the contrary, if $\psi(n_i)$ is convex in $n_i$ ($\beta < 0$), bounding $n_{i} ^ {\beta} \le 1$, we obtain \begin{align*}
    \sum_{i \in \cS} \psi(n_i) = \sum_{i \in \cS} n_{i} ^ {\beta} \comp(\cF) ^ {r} \le \abs{\cS} \comp(\cF) ^ {r} \le N \comp(\cF) ^ {r}.
\end{align*} Choosing $N = \Theta(T ^ {\frac{1}{r + 1}})$ and combining both cases yields the desired bound on $\smcal_{\Gamma, r}(\cF)$. For $r \in [1, 2)$, since $\smcal_{\Gamma, r}(\cF) \le T ^ {1 - \frac{r}{2}} (\smcal_{\Gamma, 2}(\cF)) ^ {\frac{r}{2}}$ via H\"older's inequality, we obtain  \begin{align*}
    \smcal_{\Gamma, r}(\cF) &= \tilde{\cO}\bigc{T ^ {1 - \frac{r}{2}} \cdot \bigcurl{\rho ^ {2} T^{\frac{1}{3}} + T^{\frac{1}{3}} \log \frac{1}{\delta} + T ^ {\frac{4\alpha - 1}{3}} \comp(\cF) ^ 2 + T ^ {\frac{1}{3}} \comp(\cF) ^ {2}} ^ {\frac{r}{2}}} \\
    &= \tilde{\cO}\bigc{{\rho ^ {r} T ^ {1 - \frac{r}{3}} + T ^ {1 - \frac{r}{3}} \bigc{\log \frac{1}{\delta}} ^ {\frac{r}{2}} + T ^ {1 + \frac{2r (\alpha - 1)}{3}} \comp(\cF) ^ {r} + T ^ {1 - \frac{r}{3}} \comp(\cF) ^ {r}}},
\end{align*}
where the last equality follows since $(u + v + w + x) ^ {\frac{r}{2}} \le \sqrt{4 ^ {r - 1} (u ^ {r} + v ^ {r} + w ^ {r} + x ^ {r})} \le 2 ^ {r - 1} (u ^ {\frac{r}{2}} + v ^ {\frac{r}{2}} + w ^ {\frac{r}{2}} + x ^ {\frac{r}{2}})$. This completes the proof.
\end{proof}
Observe that whenever the dependence of $n_i$ in $\reg(\cF, n_i)$ is $\sqrt{n_i}$ or better ($\alpha \le \frac{1}{2}$), for $r \ge 2$, the dependence of $T$ in the third term in \pref{thm:smcal_general_result} is at most $T^{1 - r + \frac{r}{r + 1} + \frac{r ^ {2}}{2(r + 1)}} = T ^ {\frac{2 + 2r - r ^ {2}}{2(r + 1)}} \le T^{\frac{1}{r + 1}}$, thereby resulting in $\smcal_{\Gamma, r}(\cF) = \tilde{\cO}\bigc{T^{\frac{1}{r + 1}}}$ with high probability. {Similarly, for $r \in [1, 2)$, the dependence of $T$ in the third term is at most $T ^ {1 - \frac{r}{3}}$, thereby resulting in $\smcal_{\Gamma, r}(\cF) = \tilde{\cO}\bigc{T ^ {1 - \frac{r}{3}}}$ with high probability}.  However, when the dependence of $n_{i}$ is worse than $\sqrt{n_i}$, the dependence of $T$ can be quite worse. In the next section, we consider several hypothesis classes for which either we can derive a concrete algorithm or prove an existential result that $\smcal_{\Gamma, r}(\cF) = \tilde{\cO}\bigc{T^{\frac{1}{r + 1}}}$.

\subsection{Bounds for Specific Hypothesis Classes}
We begin by implementing the online agnostic learner in \pref{def:online_weak_agnostic_learner_intro}. As already mentioned, in order to achieve the $T^{\frac{1}{r + 1}}$ dependence, we want the online agnostic learner $\owal_{i, \sigma}$ to satisfy $\mathsf{Reg}(\cF, n) = \tilde{\cO}(\sqrt{n} \comp(\cF))$. We discuss two specific classes for which one can derive an explicit algorithm satisfying the desired guarantee. For a general hypothesis class $\cF$, we show that under the assumption that the sequential Rademacher complexity of $\cF$ satisfies $\Re^{\text{seq}}(\cF, n) = \tilde{\cO}\bigc{\frac{\comp(\cF)}{\sqrt{n}}}$, there indeed exists such an algorithm.

\paragraph{Finite $\cF$.} When $\cF$ is finite, such an algorithm is possible by invoking the Multiplicative Weights Update (MWU) algorithm and predicting $q_{t}(x_{t}) = \sum_{f \in \cF} w_{t, f} f(x_{t})$, where each $f \in \cF$ corresponds to an expert and $\{w_{t, f}\}$ represents a probability distribution over the experts output by MWU.

\begin{algorithm}[!htb]
                    \caption{Online agnostic learner for finite $\cF$} 
                    \label{alg:owal_finite_F}
                    \textbf{Initialize:} Step size $\eta_t = \sqrt{\frac{\log \abs{\cF}}{t}}$, weights $w_{1, f} = \frac{1}{\abs{\cF}}$ for all $f \in \cF$;
                    \begin{algorithmic}[1]
                            \STATE \textbf{for} $t = 1, \dots, n,$
                            \STATE\hspace{3mm} Receive context $x_{t}$;
                            
                            \STATE\hspace{3mm} Predict $q_{t}(x_t) = \sum_{f \in \cF} w_{t, f} \cdot f(x_t)$ and observe $\kappa_{t}$;
                            
                            \STATE\hspace{3mm} Update $w_{t + 1, f} = \frac{\exp\bigc{\eta_t\sum_{\tau = 1} ^ {t} f(x_{\tau}) \kappa_{\tau}}}{\sum_{f \in \cF} \exp\bigc{\eta_t\sum_{\tau = 1} ^ {t} f(x_{\tau}) \kappa_{\tau}}}$ for each $f \in \cF$;
                        \end{algorithmic}
\end{algorithm}
The following result is immediate regret guarantee of MWU \citep{arora2012multiplicative}. 
\begin{lemma}\label{lem:online_agnostic_finite_F}
    \pref{alg:owal_finite_F} ensures that $\mathsf{Reg}(\cF, n) = \cO\bigc{\sqrt{n \log \abs{\cF}}}$.
\end{lemma}
Therefore, when $\cF$ is finite, instantiating each $\owal_{i, \sigma}$ with \pref{alg:owal_finite_F}, we obtain the following corollary of \pref{thm:smcal_general_result}. 
\begin{corollary}
    Fix a $r \ge 1$ and let $\cF$ be a finite class. Then, \pref{alg:hp_algorithm_efficient} with \pref{alg:owal_finite_F} as $\owal$ achieves \begin{align*}
        \smcal_{\Gamma, r}(\cF) =  \tilde{\cO}\bigc{\bigc{\rho ^ {r}  + \bigc{\log \frac{1}{\delta}} ^ {\frac{r}{2}}  +  (\log \abs{\cF}) ^ {\frac{r}{2}}} T ^ {\frac{1}{r + 1}}}
    \end{align*}
    for $r \ge 2$ with probability at least $1 - \delta$. Consequently, for $r \in [1, 2)$ it achieves \begin{align*}
        \smcal_{\Gamma, r}(\cF) = \tilde{\cO}\bigc{\bigc{\rho ^ {r}  +  \bigc{\log \frac{1}{\delta}} ^ {\frac{r}{2}}  +  \bigc{\log \abs{\cF}} ^ \frac{r}{2}} T ^ {1 - \frac{r}{3}}}
    \end{align*}
    with probability at least $1 - \delta$.
\end{corollary}

\paragraph{Linear class $\cF$.} Let $\cF = \{f_{\theta}(x) \coloneqq \ip{\theta}{x}; \theta \in \mathbb{B}_{2} ^ {d}\}$ be the set of bounded linear functions, where $\mathbb{B}_{2} ^ {d} \coloneqq \{\theta; \norm{\theta}_{2} \le 1\}$ is the unit norm ball. Then the Online Gradient Descent (OGD) algorithm (\pref{alg:owal_linear_class}) can be used to instantiate $\owal$. In \pref{alg:owal_linear_class}, $\Pi_{\mathbb{B}_2 ^ {d}}(x) \coloneqq \argmin_{y \in \mathbb{B}_{2} ^ {d}} \norm{x - y}$ represents the projection operator. Clearly, $\Pi_{\mathbb{B}_{2} ^ {d}}(x) = x$ if $x \in \mathbb{B}_{2} ^ {d}$, and $\frac{x}{\norm{x}}$ otherwise.

\begin{algorithm}[!htb]
                    \caption{Online agnostic learner for the linear class} 
                    \label{alg:owal_linear_class}
                    \textbf{Initialize:} $\theta_{1} \in \mathbb{B}_{2} ^ {d}$;
                    \begin{algorithmic}[1]
                            \STATE \textbf{for} $t = 1, \dots, n,$
                            \STATE\hspace{3mm} Receive context $x_{t}$;
                            \STATE\hspace{3mm} Predict $q_{t}(x) = \ip{\theta_{t}}{x}$;
                            \STATE\hspace{3mm} Observe $\kappa_{t}$ and update $\theta_{t + 1} = \Pi_{\mathbb{B}_{2} ^ {d}} (\theta_{t} + \eta_{t} \kappa_{t} x_{t})$, where $\eta_{t} = \frac{2}{\sqrt{t}}$;
                        \end{algorithmic}
\end{algorithm}

The following result is immediate from the regret guarantee of OGD \citep{hazan2016introduction}. 
\begin{lemma}
    \pref{alg:owal_linear_class} ensures that $\mathsf{Reg}(\cF, n) = \cO(\sqrt{n})$.
\end{lemma}

An immediate corollary of \pref{thm:smcal_general_result} for the linear class is the following.

\begin{corollary}
    Fix a $r \ge 1$ and let $\cF$ be the linear class. Then, \pref{alg:hp_algorithm_efficient} with \pref{alg:owal_linear_class} as $\owal$ achieves \begin{align*}
        \smcal_{\Gamma, r}(\cF) = \tilde{\cO}\bigc{\bigc{\rho ^ {r}  + \bigc{\log \frac{1}{\delta}} ^ {\frac{r}{2}} + 1} T^{\frac{1}{r + 1}}}
    \end{align*}
    for $r \ge 2$ with probability at least $1 - \delta$. Consequently, for $r \in [1, 2)$ it achieves \begin{align*}
        \smcal_{\Gamma, r}(\cF) = \tilde{\cO}\bigc{\bigc{\rho ^ {r}  +  \bigc{\log \frac{1}{\delta}} ^ {\frac{r}{2}} + 1} T ^ {1 - \frac{r}{3}}}
    \end{align*}
    with probability at least $1 - \delta$.
\end{corollary}

\paragraph{$\cF$ with bounded Rademacher complexity.} Similar to \cite{okoroafor2025near}, via standard learning-theoretic arguments, for a general hypothesis class $\cF$ we can derive a bound on $\mathsf{Reg}(\cF, n)$ that depends on $\Re^{\text{seq}}(\cF, n)$ by analyzing the minimax value of the online learning game between the oracle and the adversary. Recall that as per \pref{def:online_weak_agnostic_learner_intro}, at each time $t \in [n]$ in this game, the adversary reveals a context $x_{t} \in \cX$, the oracle responds with a prediction $q_{t}(x_{t}) \in [-1, 1]$, and subsequently the adversary reveals the true outcome $\kappa_{t} \in [-1, 1]$. Since the adversary is adaptive, the value of the game $\cV^{\text{seq}}(\cF, n)$ can be expressed as \begin{align}
    \cV^{\text{seq}}(\cF, n) &= \inf_{Q_1} \sup_{(x_1, \kappa_1)} \mathbb{E}_{q_1 \sim Q_1} \dots \inf_{Q_n} \sup_{(x_n, \kappa_n)} \mathbb{E}_{q_n \sim Q_n} \bigs{\sup_{f \in \cF} \sum_{t = 1} ^ {n} f(x_{t}) \kappa_{t} - q_{t}(x_t) \kappa_{t}}, \label{eq:rademacher}
\end{align}
where $Q_{t} \in \Delta(\cX \to [-1, 1])$ represents a distribution over all predictors mapping from $\cX$ to $[-1, 1]$. Since the associated loss in \eqref{eq:rademacher} is $\ell(q, \kappa) \coloneqq -q \cdot \kappa$, which is linear and $1$-Lipschitz in $q$ (for a fixed $\kappa$), it follows from \cite[Theorem 8]{rakhlin2015online} that $\cV^{\text{seq}}(\cF, n) = \cO(n\Re^{\text{seq}}(\cF, n))$. When the sequential Rademacher complexity of $\cF$ satisfies $\Re^{\text{seq}}(\cF, n) = \tilde{\cO}\bigc{\frac{\comp(\cF)}{\sqrt{n}}}$, there exists an $\owal$ that achieves $\text{Reg}(\cF, n) = \tilde{\cO}(\sqrt{n} \comp(\cF))$. Let $\alg_{\mathsf{exist}}$ be the existential algorithm in this case. \pref{thm:smcal_general_result} then implies the following corollary.

\begin{corollary}\label{cor:existence}
    Fix a $r \ge 1$ and let $\cF$ be such that $\Re^{\text{seq}}(\cF, n) = \tilde{\cO}\bigc{\frac{\comp(\cF)}{\sqrt{n}}}$. Then, \pref{alg:hp_algorithm_efficient} with $\alg_{\text{exist}}$ as $\owal$ achieves $$\smcal_{\Gamma, r}(\cF) = \tilde{\cO}\bigc{\bigc{\rho ^ {r}  + \bigc{\log \frac{1}{\delta}} ^ {\frac{r}{2}} +  \comp(\cF) ^ {r}} T^{\frac{1}{r + 1}}}$$ for $r \ge 2$ with probability at least $1 - \delta$. Consequently, for $r \in [1, 2)$ it achieves \begin{align*}
        \smcal_{\Gamma, r}(\cF) = \tilde{\cO}\bigc{\bigc{\rho ^ {r}  +  \bigc{\log \frac{1}{\delta}} ^ {\frac{r}{2}} + \comp(\cF) ^ {r}} T ^ {1 - \frac{r}{3}}}
    \end{align*}
    with probability at least $1 - \delta$.
\end{corollary}

\bibliographystyle{alpha}
\bibliography{references}

\appendix
\section{Appendix}
\begin{lemma}[Freedman's Inequality \citep{beygelzimer2011contextual}]\label{lem:Freedman} Let $X_{1}, \dots, X_{n}$ be a martingale difference sequence where $\abs{X_{i}} \le B$ for all $i = 1, \dots, n$, and $B$ is a fixed constant. Define $V \coloneqq \sum_{i = 1} ^ {n} \mathbb{E}_{i}[X_{i}^2]$. Then, for any fixed $\mu \in \bigs{0, \frac{1}{B}}, \delta \in [0, 1]$, with probability at least $1 - \delta$ we have \begin{align*}
    \abs{\sum_{i = 1} ^ {n} X_{i}} \le \mu V + \frac{\log \frac{2}{\delta}}{\mu}.
\end{align*}
\end{lemma}
\begin{lemma}\label{lem:swap_impossible_inequality}
    There does not exist a function $f_{q}: [0, 1] \times [0, 1] \to \mathbb{R}$ such that \begin{align}\label{eq:swap_impossible_inequality}
        \sum_{p \in [0, 1]} n_{p} \bigc{q - \frac{1}{n_p} \sum_{t = 1} ^ {T} \ind{p_t = p, y_{t} \le p}} ^ {2} \le \zeta\bigc{\sup_{\nu: [0, 1] \to [0, 1]} \sum_{t = 1} ^ {T} f_{q}(p_t, y_t) - f_{q}(\nu(p_t), y_t)} + C
    \end{align}
    holds for every $T \in \mathbb{N}, q \in [0, 1]$, and sequences $\{y_t\}, \{p_t\}$ for some strictly increasing invertible function $\zeta$ with $\zeta(0) = 0$, and an absolute constant $C > 0$. 
\end{lemma}
\begin{proof}
    Assume on the contrary, that there exists a function $f_{q}$ such that \eqref{eq:swap_impossible_inequality} holds for all $q \in [0, 1]$ and sequences $\{y_{t}\}_{t = 1} ^ {T}, \{p_{t}\}_{t = 1} ^ {T}$. Fix a $q \in (0, 1), p \in [0, 1]$, $y \in (0, 1)$, and let $p_{t} = p, y_t = y$ for all $t \in [T]$. Then, \eqref{eq:swap_impossible_inequality} simplifies to \begin{align*}
        T \bigc{q - \ind{y \le p}} ^ {2} \le  \zeta(T (f_{q}(p, y) - \inf_{p^\star \in [0, 1]} f_{q}(p^\star, y))) + C.
    \end{align*}
    Since $\zeta$ is strictly increasing and invertible, this is equivalent to \begin{align*}
        f_{q}(p, y) - \inf_{p^\star \in [0, 1]} f_{q}(p^\star, y) \ge \frac{1}{T}\zeta^{-1}\bigc{T \bigc{q - \ind{y \le p} ^ {2}} - C}.
    \end{align*}
    Define the function $\psi(y) \coloneqq \inf_{p^\star \in [0, 1]} f_{q}(p^\star, y)$. Then, the inequality above implies that \begin{align*}
        \psi(y) &\ge \psi(y) + \frac{1}{T}\inf_{p^\star \in [0, 1]}\zeta^{-1}\bigc{T \bigc{q - \ind{y \le p^\star} ^ {2}} - C} \\
        &= \psi(y) + \frac{1}{T} \min\bigc{\zeta^{-1}\bigc{Tq ^ {2} - C}, \zeta^{-1}\bigc{T(1 - q)^2 - C}}.
    \end{align*}
    Let $T \in \mathbb{N}$ be such that $Tq ^ {2} \ge C$ and $T(1 - q) ^ {2} \ge C$, so that both  terms $\zeta^{-1}\bigc{Tq ^ {2} - C}$ and $\zeta^{-1}\bigc{T(1 - q)^2 - C}$ are non-negative. It follows from the inequality above that either $\zeta^{-1}(Tq ^ {2} - C) = 0$ or $\zeta^{-1}(T(1 - q) ^ 2 - C) = 0$. However, since the only solution to $\zeta^{-1}(x) = 0$ is $x = 0$, either $C = Tq^{2}$ or $C = T(1 - q)^{2}$. Clearly, this is a contradiction since $C$ is an absolute constant, therefore cannot depend on $T$. This completes the proof.
\end{proof} 
\end{document}

%% file: command_preamble.tex
\newcommand{\abs}[1]{\ensuremath{\left\lvert#1\right\rvert}}
\renewcommand{\c}[1]{\ensuremath{\mathcal{#1}}} 
\newcommand{\norm}[1]{\ensuremath{\left\|#1\right\|}} 

\newcommand{\ip}[2]{\ensuremath{\left\langle #1, #2 \right\rangle}} 

\newcommand{\reg}{\mathsf{Reg}}

\newcommand{\sreg}{\mathsf{SReg}}

\newcommand{\owal}{\mathsf{OAL}}
\newcommand{\comp}{\mathsf{Comp}}
\newcommand{\bigs}[1]{\left[ #1 \right]}
\newcommand{\bigc}[1]{\left( #1 \right)}
\newcommand{\bigcurl}[1]{\left\{ #1 \right\}}
\newcommand{\ind}[1]{\mathbb{I}{\left[#1\right]}}
\newcommand{\pref}[1]{\prettyref{#1}}

\DeclareMathOperator*{\argmin}{argmin}

\definecolor{lightblue}{RGB}{0,102,204}

\def \x {{\bi{x}}}

\def \x{\boldsymbol{x}}

\def \cA {{\c{A}}}

\def \cF {{\c{F}}}
\def \cX {{\c{X}}}
\def \cY {{\c{Y}}}
\def \O {{\c{O}}}

\def \Rn {{\mathbb{R}}}

\def \cA{\mathcal{A}}

\def \cF{\mathcal{F}}
\def \cG{\mathcal{G}}

\def \cI{\mathcal{I}}
\def \cJ{\mathcal{J}}

\def \cO{\mathcal{O}}
\def \cP{\mathcal{P}}

\def \cS{\mathcal{S}}

\def \cV{\mathcal{V}}

\def \cX{\mathcal{X}}
\def \cY{\mathcal{Y}}
\def \cZ{\mathcal{Z}}

\def \alg{\mathsf{ALG}}

\def \cal{\mathsf{Cal}}

\def \smcal{\mathsf{SMCal}}

\def \mcal{\mathsf{MCal}}

\newcommand{\savehyperref}[2]{\texorpdfstring{\hyperref[#1]{#2}}{#2}}
\newrefformat{eq}{\savehyperref{#1}{Eq.~\eqref{#1}}}
\newrefformat{lem}{\savehyperref{#1}{Lemma~\ref*{#1}}}
\newrefformat{def}{\savehyperref{#1}{Definition~\ref*{#1}}}
\newrefformat{line}{\savehyperref{#1}{line~\ref*{#1}}}
\newrefformat{thm}{\savehyperref{#1}{Theorem~\ref*{#1}}}
\newrefformat{corr}{\savehyperref{#1}{Corollary~\ref*{#1}}}
\newrefformat{cor}{\savehyperref{#1}{Corollary~\ref*{#1}}}
\newrefformat{sec}{\savehyperref{#1}{Section~\ref*{#1}}}
\newrefformat{subsec}{\savehyperref{#1}{Subsection~\ref*{#1}}}
\newrefformat{tab}{\savehyperref{#1}{Table~\ref*{#1}}}
\newrefformat{app}{\savehyperref{#1}{Appendix~\ref*{#1}}}
\newrefformat{ass}{\savehyperref{#1}{Assumption~\ref*{#1}}}
\newrefformat{ex}{\savehyperref{#1}{Example~\ref*{#1}}}
\newrefformat{fig}{\savehyperref{#1}{Figure~\ref*{#1}}}
\newrefformat{alg}{\savehyperref{#1}{Algorithm~\ref*{#1}}}
\newrefformat{rem}{\savehyperref{#1}{Remark~\ref*{#1}}}
\newrefformat{item}{\savehyperref{#1}{Item~\ref*{#1}}}
\newrefformat{conj}{\savehyperref{#1}{Conjecture~\ref*{#1}}}
\newrefformat{prop}{\savehyperref{#1}{Proposition~\ref*{#1}}}
\newrefformat{line}{\savehyperref{#1}{Line~\ref*{#1}}}

\newrefformat{step}{\savehyperref{#1}{Step~\ref*{#1}}}
\newrefformat{subsubsec}{\savehyperref{#1}{Subsubsection~\ref*{#1}}}
\newrefformat{proto}{\savehyperref{#1}{Protocol~\ref*{#1}}}
\newrefformat{prob}{\savehyperref{#1}{Problem~\ref*{#1}}}
\newrefformat{claim}{\savehyperref{#1}{Claim~\ref*{#1}}}
\newrefformat{que}{\savehyperref{#1}{Question~\ref*{#1}}}

\newtheorem{assumption}{Assumption}
\newtheorem{remark}{\bf Remark}
\newtheorem{theorem}{Theorem}
\newtheorem{corollary}{Corollary}
\newtheorem{lemma}{Lemma}
\newtheorem{definition}{Definition}

\newtheorem{proposition}{Proposition}

\usepackage{tcolorbox}  

\newtcolorbox{schemebox}{
  colback=gray!10,      
  colframe=black,       
  boxrule=0.5pt,        
  arc=0mm,              
  fonttitle=\bfseries,  
  left=2mm, right=2mm,  
  top=2mm, bottom=2mm   
}